\documentclass[twoside]{article}
\usepackage[accepted]{aistats2024}

\usepackage[round]{natbib}

\usepackage[utf8]{inputenc} 
\usepackage[T1]{fontenc}    
\usepackage{hyperref}       
\usepackage{url}            
\usepackage{subfigure}
\usepackage{booktabs}       
\usepackage{amsfonts}       
\usepackage{nicefrac}       
\usepackage{microtype}      
\usepackage{xcolor}         

\usepackage{mathtools}
\usepackage{float}
\usepackage{amsmath,amsfonts,amssymb,amsthm}
\usepackage{thmtools,thm-restate}
\usepackage{bm}
\usepackage{enumitem,array}
\usepackage{algorithmic}
\usepackage{algorithm}
\usepackage{bbm}
\usepackage{tabulary,multirow}
\usepackage{dsfont}
\usepackage{thmtools,thm-restate}
\usepackage[capitalize,noabbrev]{cleveref}

\usepackage{multicol}
\usepackage{authblk}
\usepackage{upgreek}
\usepackage{tabularx}
\usepackage{lipsum}
\usepackage{enumitem}
\usepackage{authblk}

\usepackage[font=small,labelfont=bf]{caption}

\usepackage{wrapfig}
\usepackage{lipsum}

\usepackage[frozencache,cachedir=.]{minted}
\definecolor{LightGray}{gray}{0.9}

\usepackage[acronym]{glossaries}
\hypersetup{%
    colorlinks=true,
    linkcolor=mydarkblue,
    citecolor=mydarkblue,
    filecolor=mydarkblue,
    urlcolor=mydarkblue,
}
\definecolor{mydarkblue}{rgb}{0,0.08,0.45}
\makeglossaries
\newacronym{AI}{AI}{Artificial Intelligence}
\newacronym{IID}{IID}{Independent-identically-distributed}
\newacronym{OOD}{OOD}{Out-of-Distribution}
\newacronym{DNN}{DNN}{Deep Neural Network}
\newacronym{ERM}{ERM}{Empirical Risk Minimization}
\newacronym{MLE}{MLE}{Maximum Likelihood Estimation}
\newacronym{BNN}{BNN}{Bayesian Neural Network}
\newacronym{SNGP}{SNGP}{Spectral-normalized Neural Gaussian Process}
\newacronym{GP}{GP}{Gaussian Process}
\newacronym{DUE}{DUE}{Deterministic Uncertainty Estimation}
\newacronym{NatPN}{NatPN}{Natural Posterior Network}
\newacronym{MC}{MC}{Monte-Carlo}
\newacronym{MFVI}{MFVI}{Mean-Field Variational Inference}
\newacronym{EDL}{EDL}{Evidential Deep Learning}
\newacronym{NLL}{NLL}{Negative log-likelihood}
\newacronym{RMSE}{RMSE}{Root Mean Square Error}
\newacronym{SOTA}{SOTA}{State-of-the-art}

\theoremstyle{plain}
\newtheorem{theorem}{Theorem}[section]

\newtheorem{lemma}[theorem]{Lemma}
\newtheorem{corollary}[theorem]{Corollary}
\theoremstyle{definition}
\newtheorem{definition}[theorem]{Definition}

\theoremstyle{remark}
\newtheorem{remark}[theorem]{Remark}

\usepackage[textsize=tiny]{todonotes}

\usepackage{amssymb}
\usepackage{pifont}
\newcommand{\cmark}{\ding{51}}%
\newcommand{\xmark}{\ding{55}}%

\begin{document}
\runningtitle{Efficient and Distance-Aware Deep Regressor for Uncertainty Estimation under Distribution Shifts}
\twocolumn[\aistatstitle{Density-Regression: Efficient and Distance-Aware Deep Regressor for Uncertainty Estimation under Distribution Shifts}
\aistatsauthor{ Ha Manh Bui \And Anqi Liu }
\aistatsaddress{ Department of Computer Science, Johns Hopkins University, Baltimore, MD, USA } 
]
\begin{abstract}
Morden deep ensembles technique achieves strong uncertainty estimation performance by going through multiple forward passes with different models. This is at the price of a high storage space and a slow speed in the inference (test) time. To address this issue, we propose Density-Regression, a method that leverages the density function in uncertainty estimation and achieves fast inference by a single forward pass. We prove it is distance aware on the feature space, which is a necessary condition for a neural network to produce high-quality uncertainty estimation under distribution shifts. Empirically, we conduct experiments on regression tasks with the cubic toy dataset, benchmark UCI, weather forecast with time series, and depth estimation under real-world shifted applications. We show that Density-Regression has competitive uncertainty estimation performance under distribution shifts with modern deep regressors while using a lower model size and a faster inference speed.
\end{abstract}
\section{Introduction}
Improving the uncertainty quality of \acrfull{DNN} is crucial in high-stakes \acrfull{AI} applications in real-world applications~\citep{tran2022plex,nado2021uncertainty}. For example, in regression tasks like predicting temperature in weather forecasts, depth estimation in medical diagnosis, and autonomous driving, accurate uncertainty hinges on the calibration property~\citep{guo2017on}, i.e., the frequency of realizations below specific quantiles must match the respective quantile levels~\citep{kuleshov2018accurate}. Furthermore, the forecast must exhibit an appropriate level of sharpness, i.e., concentrated around the realizations and leveraging the information in the inputs effectively~\citep{kuleshov2022sharpness}.

\begin{table}[t!]
    \centering
    \scalebox{0.85}{
    \begin{tabular}{cccc}\\
    \toprule  
    Method & $\begin{matrix}
    \text{Uncertainty} \\
    \text{quality}
    \end{matrix}$ & $\begin{matrix}
    \text{Test-time} \\
    \text{efficiency}
    \end{matrix}$ & $\begin{matrix}
    \text{Without prior} \\
    \text{requirement}
    \end{matrix}$ \\\midrule
    Deterministic & \xmark & \cmark & \cmark  \\
    Bayesian & \cmark & \cmark & \xmark  \\
    Ensembles & \cmark & \xmark & \cmark\\
    Ours & \cmark & \cmark & \cmark  \\
    \bottomrule
    \end{tabular}}
    \caption{A comparison between methods in terms of uncertainty quality (calibration \& sharpness), test-time efficiency (lightweight \& fast), and whether pre-defined prior hyper-parameters are required.}
    \label{tab:teaser}
\end{table}

\begin{figure*}[t!]
\begin{center}
\includegraphics[width=1.0\linewidth]{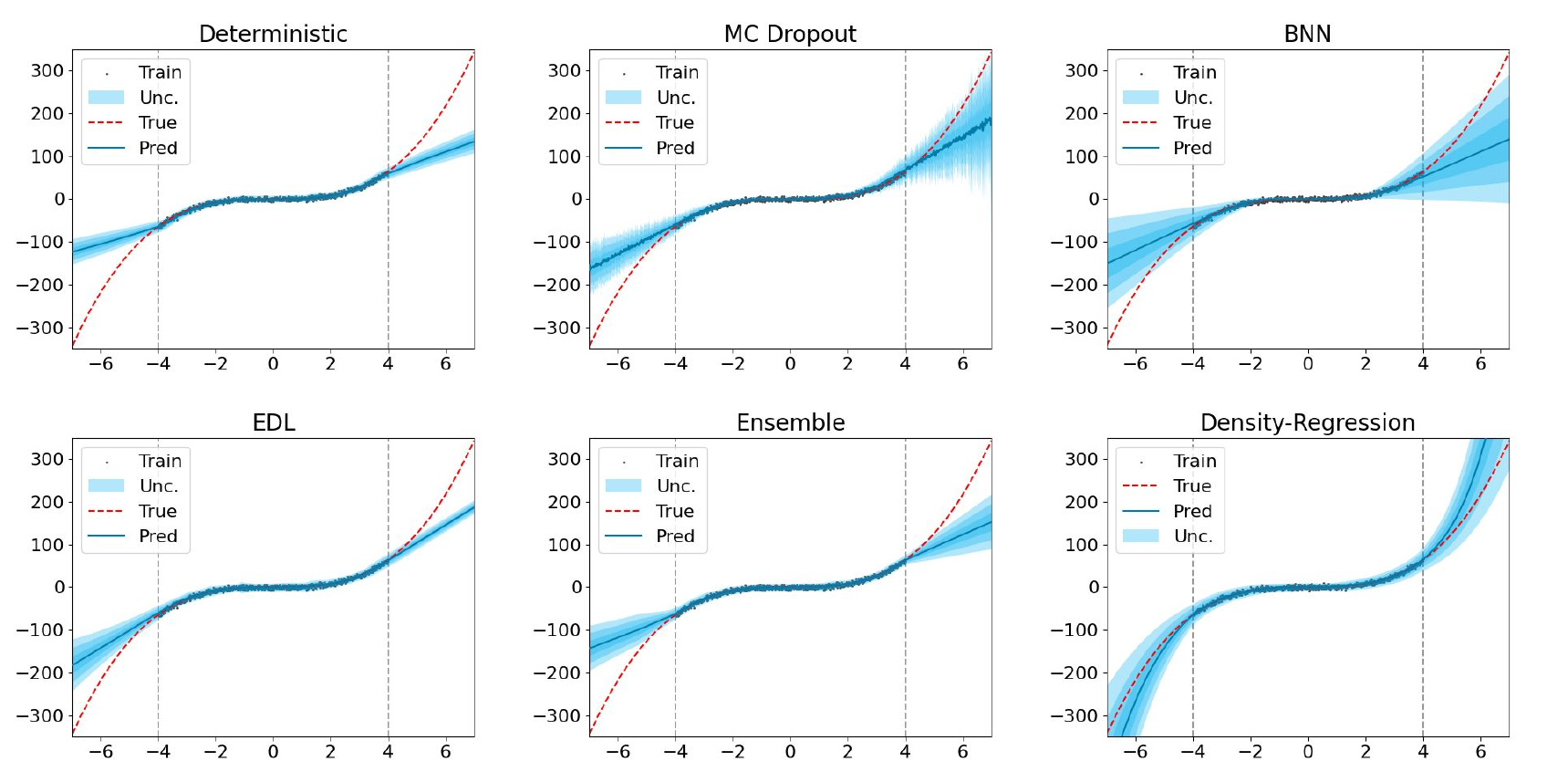}
\end{center}
  \caption{Predictive distributions for the toy dataset $y=x^3 + \epsilon$, $\epsilon \sim \mathcal{N}(0,3^2)$. The \textcolor{gray}{gray dots} in the area between \textcolor{gray}{two vertical dashed lines} represent observations in training, the \textcolor{red}{red dashed line} represents the true data-generating function, and the \textcolor{blue}{blue line} represents the mean predictions in which \textcolor{blue}{blue areas} correspond to $\pm3$ standard deviation around the mean. Our Density-Regression achieves distance awareness and, therefore, can improve distribution calibration by confident \& sharp predictions on \acrshort{IID} training data and decreased certainty and sharpness when the \acrshort{OOD} data is far from the training set. \textit{A quick demo is available at \href{https://colab.research.google.com/drive/1p5gK-rOI4XYgg2zTVtbh-5Ky06PGlA09?usp=sharing}{this Google Colab link}.}}
\label{fig:demo}
\end{figure*}

However, modern Deterministic~\acrshort{DNN} is often over-confident, especially under distribution shifts in real-world applications~\citep{tran2022plex, minderer2021revisiting, bui2021exploiting}. For instance, a Deterministic~\acrshort{DNN} trained with in-door images but deployed in outdoor scenes will result in sharp but un-calibrated predictions. Sampling-based approaches such as \acrfull{GP}, \acrfull{BNN}, \acrfull{MC}~Dropout, and Deep Ensembles can reduce over-confidence~\citep{koh2021wilds,lakshminarayanan2017ensemble,Chen2016robust}. However, such approaches usually have high computational demands by using multiple forward passes (or merging predictions from multiple models with Ensembles) at inference (test) time. 

To mitigate the inefficiency issue, sampling-free approaches, including Quantile Regression~\citep{romano2019conformalized}, \acrfull{SNGP}~\citep{Liu2020SNGP}, \acrfull{DUE}~\citep{vanamersfoort2022feature}, and closed-form posteriors in Bayesian inference (e.g., \acrfull{EDL}~\citep{sensoy2018evidential}, \acrfull{NatPN}~\citep{charpentier2022natural}), have been proposed. The Bayesian approach, however, requires pre-defined prior hyper-parameters, which are often sensitive and unknown in the real world. In addition, these aforementioned approaches still under-perform relative to sampling-based techniques and require a higher computational demand than the Deterministic~\acrshort{DNN} (e.g., Tab.~\ref{tab:teaser}). 

Towards a model enhancing uncertainty quality, test-time efficiency, and not requiring any pre-defined prior hyper-parameter, we propose Density-Regression. Our framework includes three main components: a feature extractor, a density function on feature space, and a regressor. The key component is the density function. By combining its likelihood value as an in-out-distribution detector on the feature space, the regressor achieves improved predictive uncertainty under distribution shifts. At the same time, it preserves the level of test-time efficiency of Deterministic~\acrshort{DNN}.

\textbf{Our contributions can be summarized as}:
\begin{enumerate} [leftmargin=13pt,topsep=0pt,itemsep=0mm]
    \item We introduce Density-Regression, a novel deterministic framework that improves \acrshort{DNN} uncertainty by a combination of density function with the regressor. Density-Regression is fast and lightweight. It is able to produce reliable predictions under distribution shifts, and can be implemented efficiently and easily across \acrshort{DNN} architectures. 
    \item We develop a rigorous theoretical connection for our framework, and formally prove that it is distance-aware on the feature space, i.e., its associated uncertainty metrics are monotonic functions of feature distance metrics. This is an important property to help \acrshort{DNN} improve calibration, however, it is often not guaranteed for typical \acrshort{DNN} models~\citep{Liu2020SNGP} (e.g., Fig.~\ref{fig:demo}). 
    \item We empirically show our framework achieves competitive uncertainty estimation quality with \acrfull{SOTA} across different tasks, including the cubic dataset, time-series weather forecast, UCI benchmark dataset, and monocular depth estimation. Importantly, it requires only a single forward pass and a lightweight feature density function. Therefore, it has fewer parameters and is much faster than other baselines at test time.
\end{enumerate}
\section{Background}\label{sec:background}
\subsection{Preliminaries}
\textbf{Notation and Problem setting.} Let $\mathcal{X}$ and $\mathcal{Y}$ be the sample and label space. Denote the set of joint probability distributions on $\mathcal{X} \times \mathcal{Y}$ by $\mathcal{P}_{\mathcal{X} \times \mathcal{Y}}$. A dataset is defined by a joint distribution $\mathbb{P}(x,y) \in \mathcal{P}_{\mathcal{X} \times \mathcal{Y}}$, and let $\mathcal{P}$ be a measure on $\mathcal{P}_{\mathcal{X} \times \mathcal{Y}}$, i.e., whose realizations are distributions on $\mathcal{X} \times \mathcal{Y}$. Denote the training set by $D_s = \{(x_s^i, y_s^i)\}_{i=1}^{n_s}$, where $n_s$ is the number of data points in $D_s$, s.t., $(x_s, y_s) \sim \mathbb{P}_s(x,y)$ and $\mathbb{P}_s(x,y) \sim \mathcal{P}$. In the standard learning setting, a learning model that is only trained on $D_s$, arrives at a good generalization performance on the test set $D_t = \{(x_t^i, y_t^i)\}_{i=1}^{n_t}$, where $n_t$ is the number of data points in $D_t$, s.t., $(x_t, y_t) \sim \mathbb{P}_t(x,y)$ and $\mathbb{P}_t(x,y) \sim \mathcal{P}$. In the \acrfull{IID} setting, $\mathbb{P}_{t}(x,y)$ is similar to $\mathbb{P}_s(x,y)$, and let us use $\mathbb{P}_{iid}(x,y)$ to represent the \acrshort{IID} test distribution. In contrast, $\mathbb{P}_{t}(x,y)$ is different with $\mathbb{P}_s(x,y)$ if $D_t$ is \acrfull{OOD} data, and let us use $\mathbb{P}_{ood}(x,y)$ to represent the \acrshort{OOD} test distribution.

\textbf{Deterministic Regression.} In the regression setting of representation learning, we predict a target $y \in \mathcal{Y}$, where $\mathcal{Y} = \mathbb{R}$ is continuous by using a forecast $h = g \circ f$ which composites a features extractor $f: \mathcal{X} \rightarrow \mathcal{Z}$, where $\mathcal{Z}$ is feature space, a regressor $g: \mathcal{Z} \rightarrow  \mathcal{Y}$ which outputs a predicted output over $\mathcal{Y}$. In Deterministic~\acrshort{DNN}, we often aim to learn a function $h$ by minimizing the Mean Squared Error
\begin{align}
    \min_{\theta_{g,f}}\left\{\mathbb{E}_{(x,y)\sim D_s}\left[\frac{1}{2} ||y - g(f(x))||^2\right]\right\},
\end{align}
where $\theta_{g,f}$ is the parameter of encoder $f$ and regressor $g$. The model $h$ is optimized such that it can learn the average correct answer for a given input. Yet, there is no uncertainty estimation in this model when making the prediction~\citep{chua2018deepRL,tran2020methods}.

\textbf{Deterministic~Gaussian~\acrshort{DNN}.} To tackle this issue, one probabilistic approach is based on the assumption that the true label $y$ is distributed according to a normal distribution with a true mean $\mu(x)$ and some noise with the variance $\sigma^2(x)$~\citep{chua2018deepRL}. Under this assumption, we can simplify the model to a probabilistic function $h$ to infer $(\mu,\sigma^2)$ by making $g: \mathcal{Y} \rightarrow \mathbb{R}^2$, then optimize function $h$ by \acrfull{MLE}
\begin{align}\label{eq:MLE}
    \min_{\theta_{g,f}}\{\mathbb{E}_{(x,y)\sim D_s}[-\log p\left(y|g(f(x))\right)\\\nonumber := \frac{1}{2} \log(2\pi\sigma^2) + \frac{(y-\mu)^2}{2\sigma^2} ]\}.
\end{align}
This approach, however, only helps the model $h$ provide data uncertainty when making predictions since it estimates the underlying noise in the data (i.e., aleatoric uncertainty), and the model uncertainty (i.e., epistemic uncertainty) is ignored~\citep{tran2020methods}.

\textbf{Sampling-based Regression.} The sampling-based approaches, i.e., make an inference by multiple forward passes (or merging predictions from multiple models), such as \acrshort{BNN}, \acrshort{MC}~Dropout, and Deep Ensembles can model the model uncertainty by predicting
\begin{align}\label{eq:sampling-based}
    &\mu(x)=\frac{1}{M}\sum_{i=1}^M \mu_{\theta_i}(x),\\
    &\sigma^2(x) = \frac{1}{M}\sum_{i=1}^M \left[\left(\mu(x) - \mu_{\theta_i}(x))^2 + \sigma^2_{\theta_i}(x\right)\right],
\end{align}
where $\theta_i$ represents $i$-th model's parameters. However, this approach struggles with computational cost by requiring multiple (i.e., $M$) model inferences. 

\subsection{Evaluating Uncertainty}
\textbf{Calibrated Regression.} First, we present the definition of distribution calibration for the forecast $h$ under the regression setting by:
\begin{definition}~\citep{gneiting2007probabilistic}\label{def:calibration}
A forecast $h$ is said to be {\bf distributional calibrated} if and only if
\begin{equation}
    \mathbb{P}(Y \leq F_x^{-1}(p)) = p, \forall p \in [0,1],
\end{equation}
where we use $F_x: \mathcal{Y} \rightarrow [0,1]$ to denote the CDF of forecast $h(x)$ at $x$, hence $F_i^{-1}: [0,1] \rightarrow \mathcal{Y}$ means the quantile function $F_i^{-1}(p) = \inf \{y: p \leq F_i(y)\}$.
\end{definition}
Intuitively, this means that a $p$ confidence interval contains the target $y$ $p$ of the time. This definition also implies that
$\left(\sum_{i=1}^n \mathbb{I}\{F_i^{-1}(p_1)\leq y_i \leq F_i^{-1}(p_2)\}\right)/n \rightarrow p_2 - p_1$,
for all $p_1, p_2 \in [0,1]$ as $n\rightarrow \infty$. Under this confidence intervals intuition, \citet{kuleshov2018accurate} propose to measure the calibration error as a numerical score describing the quality of forecast calibration
\begin{align}
     &cal(\{F_i,y_i\}_{i=1}^n) \notag\\&:= \sum_{j=1}^m \left(p_j - \frac{\left|\{y_i|F_i(y_i) \leq p_j, i=1,\cdots,n\}\right|}{n}\right)^2,
\end{align}
for each threshold $p_j$ from the chosen of $m$ confidence level $0\leq p_1<p_2<\cdots<p_m\leq 1$. 

\textbf{Sharpness.} Calibration, however, is only a necessary condition for good uncertainty estimation. For example, a well-calibrated model could still have large confidence intervals, which is inherently less useful than a well-calibrated one with small confidence intervals. Therefore, another condition is that the forecast $h$ must also be sharp. Intuitively, this means that the confidence intervals should be as tight as possible, i.e., $var(F_i)$ of the random variable whose CDF is $F_i$ to be small. Formally, the sharpness score~\citep{tran2020methods} follows
\begin{align}
    sha(F_1,\cdots,F_n) := \sqrt{\frac{1}{n} \sum_{i=1}^n var(F_i)}.
\end{align}

\textbf{Distance Awareness.} One of the approaches to help the Deterministic~Gaussian~\acrshort{DNN} provide model uncertainty is making it achieve distance awareness. This is an important property to improve model uncertainty under distributional shifts~\citep{Liu2020SNGP,vanamersfoort2022feature,bui2023densitysoftmax}. It was introduced by~\citet{Liu2020SNGP}, and on the feature space $\mathcal{Z}$, we can define distance awareness as follows:
\begin{definition}\label{def:distanceaware} The forecast $h(x_t)$ on the new test feature $z_{t}=f(x_t)$, is said to be \textbf{feature distance-aware} if there exists $u(z_{t})$, a summary statistics of $h(x_t)$, that quantifies model uncertainty (e.g., entropy, predictive variance, etc.) and reflects the distance between $z_{t}$ and the features random variable on the training data $Z_s$ w.r.t. a metric $\left \| \cdot \right \|_{\mathcal{Z}}$, i.e., $u(z_{t}) := v(d(z_{t},Z_s))$, where $v$ is a monotonic function and $d(z_{t},Z_s) := \mathbb{E}\left \| z_{t} - Z_s \right \|_{\mathcal{Z}}$ is the distance between $z_{t}$ and $Z_s$.
\end{definition}
Following Def.~\ref{def:distanceaware}, if a model achieves the distance-aware property, model uncertainty quality would be improved and the over-confidence issues of current \acrshort{DNN} on \acrshort{OOD} would be reduced. While, at the same time, it will still preserve certainty predictions for the \acrshort{IID} test example, suggesting calibration and sharpness improvement.

\subsection{Test-time Efficiency}
\textbf{Latency.} To deploy in real-world applications, a necessary condition is that the model must infer fast at the test-time. This is especially important in high-stakes applications such as autonomous driving when the car needs to react in sudden circumstances.

\textbf{Parameters.} Additionally, to make the model scalable, the model also needs to be as lightweight as possible to be installed with low-resource hardware. That said, current \acrshort{SOTA} sampling-based approaches that can improve uncertainty quantification are still struggling with these two criteria~\citep{tran2022plex,nado2021uncertainty,bui2023densitysoftmax}.
\section{Density-Regression}
\begin{figure}[ht!]
    \centering
    \includegraphics[width=1.0\linewidth]{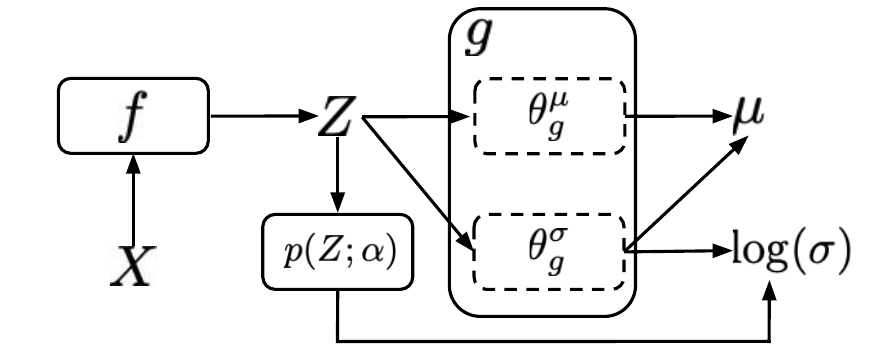}
    \caption{The overall architecture of Density-Regression, including encoder $f$, regressor $g$, and density function $p(Z;\alpha)$. Solid rectangle boxes represent these functions. Dashed rectangle boxes represent function weights. Three training steps and inference process follow Alg.~\ref{alg:algorithm}}.
    \label{fig:framework}
\end{figure}
\subsection{Exponential Family Distribution}
To improve the uncertainty quality and test-time efficiency of \acrshort{DNN} in regression under distribution shifts, we propose the Density-Regression framework in Fig.~\ref{fig:framework}. First, recall that when $\mathcal{Y}$ is the label space and $\Phi(x,y)$ is the sufficient statistic of the joint distribution $\mathbb{P}(x,y)$ associated with $(x,y)\in \mathcal{X}\times \mathcal{Y}$, from Lemma~\ref{lemma:exp} in our Appendix~\ref{apd:proof}, we have the predictive distribution $p(y|x;\theta)$ follows
\begin{align}
    p(y|x;\theta)=\frac{\exp(\eta(\theta_g)^\top\Phi(f(x),y))}{\int_{y'\in \mathcal{Y}}\exp\left(\eta(\theta_g)^\top\Phi(f(x),y')\right)dy'},
\end{align}
where $\eta$ is the natural function for the parameter $\theta$, $\theta = (\theta_g, \theta_f)$ is obtained by maximizing the \acrshort{MLE} w.r.t. the conditional log-likelihood
\begin{align}
    \theta = \underset{\theta}{\arg\max} \mathbb{E}_{x,y\sim p(x,y)}[\log p(y|x;\theta)].
\end{align}

Motivated by the idea of using the density function to improve uncertainty estimation of \acrshort{DNN} under distribution shifts~\citep{bui2023densitysoftmax,charpentier2020postnet,kuleshov2022sharpness}, in this paper, we integrate the density function into our Density-Regression's framework by
\begin{align}\label{eq:exp}
    p(y|x;\theta) = \frac{\exp\left(-p(z;\alpha)\theta_{g}^\top \Phi(z,y)\right)}{\int_{y'\in \mathcal{Y}}\exp\left(-p(z;\alpha)\theta_{g}^\top \Phi(z,y')\right)dy'}, 
\end{align}
where $z = f(x)$ and the density function $p(z;\alpha) \in (0,\infty]$.
The density $p(z;\alpha)$ modulates the distribution's certainty. When the density goes to zero, the estimator $p(y|x;\theta)$ becomes uncertain. As the density goes toward infinity, the estimator converges to a deterministic point estimate. When Density-Regression follows a Gaussian distribution, we obtain the mean and variance by the following theorem:
\begin{theorem}\label{theo:gauss} If the predictive distribution follows Eq.~\ref{eq:exp} and the sufficient statistic has the form $\Phi(z,y) = \begin{bmatrix}zy^2 & y^2 & 2zy & 2y & z & 1\end{bmatrix},$ then Density-Regression has the conditional Gaussian distribution as follows $p(y|x;\theta) \sim \mathcal{N}(\mu(x,\theta),\sigma^2(x,\theta)),$
where
    \begin{align*}
        &\mu(x,\theta) = -\left(\theta_g^\sigma\begin{bmatrix}
            z\\1
        \end{bmatrix}\right)^{-1}\left(\theta_g^\mu{\begin{bmatrix}
            z\\1
        \end{bmatrix}}\right),\\
        &\sigma^2(x,\theta) = \left(2\cdot p(z;\alpha)\theta_g^\sigma\begin{bmatrix}
            z\\1
        \end{bmatrix}\right)^{-1},
    \end{align*}
    where $z=f(x)$ and $\theta_{g}^\mu$ and $\theta_{g}^\sigma$ are the parameters (model weights) of the regressor $g$, i.e., $(\theta_{g}^\mu,\theta_{g}^\sigma)=\theta_g$ . The proof is in Apd.~\ref{proof:theo:gauss}.
\end{theorem}

The mean and variance in Theorem~\ref{theo:gauss} are hard to optimize with \acrshort{MLE} since they are undefined for non-positive $\theta_g^\sigma\begin{bmatrix}z,1\end{bmatrix}^\top$. To avoid this issue, we can rewrite the mean and the variance by the Corollary below:
\begin{corollary}\label{corol:gauss} Given the mean and variance in Theorem~\ref{theo:gauss}, we can rewrite as follows
\begin{align*}
    &\mu(x,\theta) = \sigma^2(z,\theta)\left(-2 \cdot p(z;\alpha)\theta_g^\mu{\begin{bmatrix}z\\1\end{bmatrix}}\right),\\
    &\sigma^2(x,\theta) = \exp\left(-\frac{1}{2}\left[\log(2)+\log(p(z;\alpha))+\theta_g^\sigma\begin{bmatrix}z\\1\end{bmatrix}\right]\right)^2
    \end{align*}
    , where $z=f(x)$. The proof is in Apd.~\ref{proof:corol:gauss}.
\end{corollary}
Corollary~\ref{corol:gauss} gives a defined mean and variance formulation for every $\theta_g^\sigma\begin{bmatrix}z,1\end{bmatrix}^\top \in \mathbb{R}$ such that they can be easily optimized with \acrshort{MLE}.

\subsection{Training and Inference Process}
Given these aforementioned results, we next present the training and inference of our framework:

\textbf{Training.} In the first training step, we optimize the model by using \acrfull{ERM}~\citep{vapnik1998erm} with training data $D_s$, i.e., we do \acrshort{MLE}
\begin{equation}\label{method:1st-optimization}
    \min_{\theta_{g,f}} \left \{ \mathbb{E}_{(x,y) \sim D_s}\left [ \frac{1}{2} \log(2\pi \sigma_o^2(x,\theta)) + \frac{(y - \mu_o(x,\theta))^2}{2\sigma_o^2(x,\theta)} \right ] \right\}
\end{equation}
, where $\mu_o(x,\theta), \sigma_o^2(x,\theta)$ follows Corollary~\ref{corol:gauss} without the density function $p(z;\alpha)$, i.e.,
\begin{align}
    &\mu_o(x,\theta) = \sigma_o^2(z,\theta)\left(-2 \cdot \theta_g^\mu{\begin{bmatrix}z\\1\end{bmatrix}}\right), \label{eq:mu}\\
    &\sigma_o^2(x,\theta) = \exp\left(-\frac{1}{2}\left[\log(2)+\theta_g^\sigma\begin{bmatrix}z\\1\end{bmatrix}\right]\right)^2, \label{eq:sigma}
\end{align}
where $z=f(x)$.

After that, we freeze the parameter $\theta_{f}$ of $f$ to estimate the density on the feature space $\mathcal{Z}$ by positing a statistical model for $p(Z;\alpha)$ with Normalizing-Flows~\citep{dinh2017density} since it is provable (in terms of Lemma~\ref{lemma:density}), simple, and provides exact log-likelihood~\citep{charpentier2022natural}. Specifically, we optimize by using \acrshort{MLE} w.r.t. the logarithm as follows
\begin{align}
    &\max_{\alpha} \{ \mathbb{E}_{z=f(x)\sim D_B} [ \log(p(z;\alpha))\nonumber\\ &:= \log(p(t;\alpha)) + \log \left | \det \left ( \frac{\partial t}{\partial z}\right ) \right | ] \},
\end{align}
where random variable $t=s_{\alpha}(f(x))$ and $s$ is a bijective differentiable function. It is worth noticing that $p(Z;\alpha)$ can be any other density function (e.g., Kernel density estimation, Gaussian mixture models, etc.).

Finally, we combine our model with the likelihood value of the density function $p(Z;\alpha)$ by re-updating the weight of classifier $g$, i.e., $\theta_g$ via optimizing 
\begin{equation}\label{method:2nd-optimization}
    \min_{\theta_{g}} \left \{ \mathbb{E}_{(x,y) \sim D_s} \left [ \frac{1}{2} \log(2\pi \sigma^2(x,\theta)) + \frac{(y - \mu(x,\theta))^2}{2\sigma^2(x,\theta)} \right ] \right \}
\end{equation}
, where $\mu(x,\theta)$ and $\sigma^2(x,\theta)$ follows Corollary~\ref{corol:gauss}.

\textbf{Inference.} After completing the training process, for a new input $x_t$ at the test-time, we perform prediction by combining the density function on feature space $p(z_t;\alpha)$ following Corollary~\ref{corol:gauss}. The pseudo-code for the training and inference processes of our proposed framework is presented in Algorithm~\ref{alg:algorithm} and the demo notebook code is in Apd.~\ref{apd:code}.

\begin{figure}[t!]
\begin{algorithm}[H]
\caption{Training and Inference (code is in Apd.~\ref{apd:code})}
\label{alg:algorithm}
\begin{algorithmic}
    \STATE \textbf{Training Input}: Dataset $D_s$, encoder $f$, density $p(f(X);\alpha)$, regressor $g$ with $\theta_g = (\theta_g^\sigma,\theta_g^\mu)$, learning rate $\eta$, batch-size $B$\;
    \FOR{$e=1\rightarrow \text{epochs}$}
        \STATE Sample $(x,y)\in D_B$ with a mini-batch $B$ for $D_s$\;
        \STATE Set $\mu_o(x,\theta)$ and $\sigma_o^2(x,\theta)$ following Eq.~\ref{eq:mu} and Eq.~\ref{eq:sigma}
        \STATE Update $\theta_{g,f}$ as:
        \STATE \begin{normalsize}\fontsize{9.5pt}{10pt}\selectfont$\theta_{g,f} -\eta \nabla_{\theta_{g,f}}\mathbb{E}_{(x,y)} \left [ \frac{1}{2} \log(2\pi \sigma_o^2(x,\theta)) + \frac{(y - \mu_o(x,\theta))^2}{2\sigma_o^2(x,\theta)} \right ]$\;\end{normalsize}
    \ENDFOR
    \FOR{$e=1\rightarrow \text{train-density epochs}$}
        \STATE Sample $x \in D_B$ with a mini-batch $B$ for $D_s$\;
        \STATE Update $\alpha$ as:
        \STATE $\alpha - \eta \nabla_{\alpha} \mathbb{E}_{z=f(x)} \left[ -\log(p(t;\alpha)) - \log \left | \det \left ( \frac{\partial t}{\partial z}\right ) \right |\right]$
    \ENDFOR
    \FOR{$e=1\rightarrow \text{epochs}$}
        \STATE Sample $(x,y)\in D_B$ with a mini-batch $B$ for $D_s$\;
        \STATE Set $\sigma^2(x,\theta)$ and $\mu(x,\theta)$ following Corollary~\ref{corol:gauss}
        \STATE Update $\theta_{g}$ as:
        \STATE $\theta_{g} -\eta \nabla_{\theta_{g}}\mathbb{E}_{(x,y)} \left [ \frac{1}{2} \log(2\pi \sigma^2(x,\theta)) + \frac{(y - \mu(x,\theta))^2}{2\sigma^2(x,\theta)} \right ]$\;
     \ENDFOR
    \STATE \textbf{Inference Input}: Test sample $x_t$\;
    \STATE Set $\sigma^2(x_t,\theta)$ and $\mu(x_t,\theta)$ following Corollary~\ref{corol:gauss}
    \STATE Predict $\hat{y}_t \sim \mathcal{N}(\mu(x_t,M);\sigma^2(x_t,M))$\;
\end{algorithmic}
\end{algorithm}
\end{figure}

\begin{remark}\label{remark:efficiency} \textbf{(Computational efficiency at test-time)}
    Corollary~\ref{corol:gauss} shows that the complexity of Density-Regression at test-time is only $\mathcal{O}(1)$ by requiring only a single forward pass. Meanwhile, the sampling-based approach is $\mathcal{O}(M)$, where $M$ is the number of sampling times. Compared with Deterministic~Gaussian~\acrshort{DNN}, ours computation only needs to additionally compute $p(z_t;\alpha)$, therefore, only higher than Deterministic~Gaussian~\acrshort{DNN} by the additional parameter $\alpha$ and the latency of $p(z_t;\alpha)$. This number is often very small in practice by Fig.~\ref{fig:model_size}~(a) and Fig.~\ref{fig:model_size}~(b).
\end{remark}

\subsection{Theoretical Analysis}
Intuitively, the predictive distribution of Density-Regression in Corollary~\ref{corol:gauss} leads to reasonable uncertainty estimation for the two limit cases of strong \acrshort{IID} and \acrshort{OOD} data. In particular, for very unlikely \acrshort{OOD} data, i.e., $d(z_{t},Z_s) \rightarrow \infty$, the prediction become uncertain. Conversely, for very likely \acrshort{IID} data, i.e., $d(z_{t},Z_s) \rightarrow 0$, the prediction converges to a deterministic point estimate. We formally show this property below:

Let us recall a Lemma when $p(Z;\alpha)$ is a Normalizing-Flows model~\citep{papamakarios2021normalizing}:
\begin{lemma}\label{lemma:density} (Lemma 5~\citep{charpentier2022natural})
    If $p(Z;\alpha)$ is parametrized with a Gaussian Mixture Model (GMM) or a radial Normalizing-Flows, then $\lim_{d(z_{t},Z_s) \rightarrow \infty} p(z_t;\alpha) \rightarrow 0$.
\end{lemma}

Based on Lemma~\ref{lemma:density}, we obtain the following results:
\begin{theorem}\label{theorem:uncertainty_iid_ood} 
Density-Regression converges to a deterministic point estimate, i.e., $\sigma^2(x_{iid};\theta) \rightarrow 0$ when $d(z_{iid},Z_s)\rightarrow 0$ by $p(z_{iid};\alpha)\rightarrow \infty$, and become uncertain, i.e., $\sigma^2(x_{ood};\theta) \rightarrow \infty$ when $d(z_{ood},Z_s)\rightarrow \infty$ by $p(z_{ood};\alpha)\rightarrow 0$. The proof is in Apd.~\ref{proof:theorem:uncertainty_iid_ood}.
\end{theorem}

\begin{theorem}\label{theorem:distance-aware} The predictive distribution of Density-Regression $p(y|x;\theta) \propto  \exp(-p(f(x);\alpha) \theta_g^\top \Phi(f(x),y))$ is distance-aware on feature space $\mathcal{Z}$ by satisfying the condition in Def.~\ref{def:distanceaware}, i.e., there exists a summary statistic $u(z_{t})$ of $p(\hat{y}_t|x_t;\theta)$ on the new test feature $z_t=f(x_t)$ s.t. $u(z_{t}) = v(d(z_{t},Z_s))$, where $v$ is a monotonic function and $d(z_{t},Z_s) = \mathbb{E}\left \| z_{t} - Z_s \right \|_{\mathcal{Z}}$ is the distance between $z_{t}$ and the training features random variable $Z_s$.
The proof is in Apd.~\ref{proof:theorem:distance-aware}.
\end{theorem}

\begin{remark}
    Theorem~\ref{theorem:distance-aware} shows our Density-Regression is distance-aware on the feature representation $\mathcal{Z}$, i.e., its predictive probability reflects monotonically the distance between the test feature and the training set. This is a necessary condition for a \acrshort{DNN} to achieve high-quality uncertainty estimation~\citep{Liu2020SNGP,bui2023densitysoftmax}. Combining with Theorem~\ref{theorem:uncertainty_iid_ood}, this proves when the likelihood of $p(Z;\alpha)$ is high, our model is certain on \acrshort{IID} data, and when the likelihood of $p(Z;\alpha)$ decreases on \acrshort{OOD} data, the certainty will decrease correspondingly.
\end{remark}
\section{Experiments}\label{sec:experiments}
\begin{table*}[ht!]
    \caption{Time series temperature weather forecasting, models are trained on data from 2009 to 2016 in the Beutenberg weather station, tested on this \acrshort{IID} data, and \acrshort{OOD} data from Saaleaue station in 2022. \acrfull{NLL}, \acrfull{RMSE}, Calibration (Cal), and Sharp (Sharpness) mean evaluation on \acrshort{IID} test-set. oNLL, oRMSE, oCal, and oSharp correspond these metrics with evaluation on \acrshort{OOD} test-set. Lower is better. Results are reported over 10 different random seeds. Best scores that have the $p\text{-value} \leq0.05$ in the significant T-test are marked in \textbf{bold}.}
    \centering
    \scalebox{0.77}{
    \begin{tabular}{lcccccccc}
    \toprule
    \textbf{Method} & \textbf{NLL ($\downarrow$)} & \textbf{RMSE ($\downarrow$)} & \textbf{Cal ($\downarrow$)} & \textbf{Sharp ($\downarrow$)} & \textbf{oNLL ($\downarrow$)} & \textbf{oRMSE ($\downarrow$)} & \textbf{oCal ($\downarrow$)} & \textbf{oSharp ($\downarrow$)}\\
    \midrule
    Deterministic & -3.90 $\pm$ 0.10 & 0.09 $\pm$ 0.01 & 0.47 $\pm$ 0.29 & 0.10 $\pm$ 0.01 & -1.76 $\pm$ 0.46 & 0.15 $\pm$ 0.01 & 0.58 $\pm$ 0.35 & 0.09 $\pm$ 0.01\\
    Quantile  & -3.67 $\pm$ 0.10 & 0.10 $\pm$ 0.01 & 0.88 $\pm$ 0.58 & 0.08 $\pm$ 0.01 & -0.11 $\pm$ 0.85 & 0.15 $\pm$ 0.00 & 1.36 $\pm$ 0.77 & 0.08 $\pm$ 0.01\\
    MC~Dropout & -3.48 $\pm$ 0.11 & 0.10 $\pm$ 0.00 & 0.81 $\pm$ 0.26 & 0.16 $\pm$ 0.01 & \textbf{-2.52 $\pm$ 0.06} & 0.16 $\pm$ 0.00 & 0.32 $\pm$ 0.24 & 0.16 $\pm$ 0.01\\
    MFVI-BNN & -3.77 $\pm$ 0.35 & 0.10 $\pm$ 0.02 & 0.73 $\pm$ 0.51 & 0.12 $\pm$ 0.03 & -2.25 $\pm$ 0.26 & 0.15 $\pm$ 0.01 & 0.54 $\pm$ 0.56 & 0.12 $\pm$ 0.03\\
    EDL & -3.96 $\pm$ 0.09  & 0.09 $\pm$ 0.00 & 0.55 $\pm$ 0.31 & 0.10 $\pm$ 0.01 & -2.19 $\pm$ 0.22 & \textbf{0.14 $\pm$ 0.00} & 0.68 $\pm$ 0.37 & 0.09 $\pm$ 0.01\\
    SNGP & -3.35 $\pm$ 0.19 & 0.13 $\pm$ 0.01 & 0.52 $\pm$ 0.47 & 0.15 $\pm$ 0.01 & -2.20 $\pm$ 0.09 & 0.27 $\pm$ 0.02 & 0.29 $\pm$ 0.25 & 0.23 $\pm$ 0.05\\
    DUE & -3.37 $\pm$ 0.22 & 0.13 $\pm$ 0.02 & 0.64 $\pm$ 0.67 & 0.14 $\pm$ 0.01 & -2.38 $\pm$ 0.13 & 0.27 $\pm$ 0.03 & 0.27 $\pm$ 0.20 & 0.20 $\pm$ 0.03\\
    Ensembles & \textbf{-4.04 $\pm$ 0.06} & \textbf{0.08 $\pm$ 0.00} & 0.62 $\pm$ 0.20 & 0.11 $\pm$ 0.01 & -2.44 $\pm$ 0.15 & \textbf{0.14 $\pm$ 0.00} & 0.25 $\pm$ 0.13 & 0.10 $\pm$ 0.01\\
    \textbf{Ours} & -3.99 $\pm$ 0.08 & 0.08 $\pm$ 0.01 & \textbf{0.38 $\pm$ 0.14} & \textbf{0.09 $\pm$ 0.00} & -2.16 $\pm$ 0.16 & \textbf{0.14 $\pm$ 0.00} & \textbf{0.23 $\pm$ 0.10} & \textbf{0.09 $\pm$ 0.00}\\
    \bottomrule
    \end{tabular}}
    \label{tab:time-series}
\end{table*}

\begin{table*}[ht!]
    \caption{UCI: Wine Quality, models are trained on red wine sample, tested on this \acrshort{IID} data, and \acrshort{OOD} white vinho verde wine samples, from the north of Portugal.}
    \centering
    \scalebox{0.78}{
    \begin{tabular}{lcccccccc}
    \toprule
    \textbf{Method} & \textbf{NLL ($\downarrow$)} & \textbf{RMSE ($\downarrow$)} & \textbf{Cal ($\downarrow$)} & \textbf{Sharp ($\downarrow$)} & \textbf{oNLL ($\downarrow$)} & \textbf{oRMSE ($\downarrow$)} & \textbf{oCal ($\downarrow$)} & \textbf{oSharp ($\downarrow$)}\\
    \midrule
    Deterministic & 1.10 $\pm$ 0.01 & 0.62 $\pm$ 0.01 & 0.69 $\pm$ 0.05 & 1.14 $\pm$ 0.02 & 1.23 $\pm$ 0.04 & 0.82 $\pm$ 0.03 & 0.49 $\pm$ 0.38 & 1.04 $\pm$ 0.03\\
    Quantile  & 1.31 $\pm$ 0.01 & 0.66 $\pm$ 0.03 & 1.30 $\pm$ 0.08 & 1.64 $\pm$ 0.03 & 1.53 $\pm$ 0.05 & 0.88 $\pm$ 0.08 & 1.59 $\pm$ 0.62 & 1.99 $\pm$ 0.13\\
    MC~Dropout & 1.14 $\pm$ 0.03 & 0.63 $\pm$ 0.01 & 0.73 $\pm$ 0.08 & 1.20 $\pm$ 0.05 & 1.27 $\pm$ 0.04 & 0.85 $\pm$ 0.04 & 0.55 $\pm$ 0.39 & 1.12 $\pm$ 0.03\\
    MFVI-BNN & 1.36 $\pm$ 0.02 & 0.79 $\pm$ 0.01 & 1.25 $\pm$ 0.08 & 1.58 $\pm$ 0.05 & 1.40 $\pm$ 0.02 & 0.91 $\pm$ 0.01 & 1.37 $\pm$ 0.10 & 1.53 $\pm$ 0.05\\
    EDL & \textbf{0.98 $\pm$ 0.03} & 0.61 $\pm$ 0.01 & 0.97 $\pm$ 0.17 & 1.37 $\pm$ 0.16 & 1.55 $\pm$ 0.18 & 0.86 $\pm$ 0.06 & 1.75 $\pm$ 1.30 & 2.87 $\pm$ 4.25\\
    SNGP & 1.18 $\pm$ 0.03 & 0.65 $\pm$ 0.01 & 0.81 $\pm$ 0.07 & 1.36 $\pm$ 0.05 & 1.42 $\pm$ 0.03 & 0.87 $\pm$ 0.02 & 1.09 $\pm$ 0.16 & 1.66 $\pm$ 0.08\\
    DUE & 1.14 $\pm$ 0.03 & 0.64 $\pm$ 0.01 & 0.77 $\pm$ 0.09 & 1.26 $\pm$ 0.05 & 1.27 $\pm$ 0.03 & 0.82 $\pm$ 0.03 & 0.75 $\pm$ 0.42 & 1.27 $\pm$ 0.05\\
    Ensembles & 1.06 $\pm$ 0.01 & \textbf{0.60 $\pm$ 0.00} & \textbf{0.68 $\pm$ 0.04} & \textbf{1.15 $\pm$ 0.01} & \textbf{1.21 $\pm$ 0.02} & \textbf{0.80 $\pm$ 0.02} & 0.44 $\pm$ 0.22 & 1.09 $\pm$ 0.02\\
    \textbf{Ours} & 1.06 $\pm$ 0.01 & \textbf{0.60 $\pm$ 0.01} & \textbf{0.68 $\pm$ 0.02} & \textbf{1.16 $\pm$ 0.00} & \textbf{1.21 $\pm$ 0.02} & 0.81 $\pm$ 0.03 & \textbf{0.43 $\pm$ 0.20} & \textbf{1.07 $\pm$ 0.01}\\
    \bottomrule
    \end{tabular}}
    \label{tab:uci}
\end{table*}

\begin{table*}[ht!]
    \caption{Monocular depth estimation, models are trained on NYU~Depth~v2 dataset, tested on this \acrshort{IID} data, and two different \acrshort{OOD} test-set. cNLL, cRMSE, cCal, and cSharp mean evaluation on the corrupted \acrshort{OOD} data dataset. oNll, oRMSE, oCal, and oSharp mean evaluation on ApolloScape \acrshort{OOD} data.}
    \centering
    \scalebox{0.618}{
    \begin{tabular}{lcccccccccccc}
    \toprule
    \textbf{Method} & \textbf{NLL ($\downarrow$)} & \textbf{RMSE ($\downarrow$)} & \textbf{Cal ($\downarrow$)} & \textbf{Sharp ($\downarrow$)} & \textbf{cNLL ($\downarrow$)} & \textbf{cRMSE ($\downarrow$)} & \textbf{cCal ($\downarrow$)} & \textbf{cSharp ($\downarrow$)} & \textbf{oNLL ($\downarrow$)} & \textbf{oRMSE ($\downarrow$)} & \textbf{oCal ($\downarrow$)} & \textbf{oSharp ($\downarrow$)}\\
    \midrule
    Deterministic & -2.4599  & 0.0381  & 0.0955 & 0.0348  & -1.9875 & 0.0624 & 0.8831 & 0.0360 & 11.8430 & 0.3579 & 27.3004 & 0.0859\\
    MC~Dropout & -2.1086  & 0.0484  & 0.2632  & 0.0530 & -1.1290 & 0.0904 & 0.7900 & 0.0650 & 11.0895 & 0.3625 & 27.4328 & 0.1050\\
    EDL & -2.3165 & 0.0373 & 0.1903 & 0.0525 & 13.3139 & 0.1339 & 2.7539 & 0.1238 & 30.3392 & 0.5392 & 27.9011 & 0.1352\\
    NatPN & -2.1854 & 0.0384 & 0.0953 & 0.0397 & -1.1154 & 0.1021 & 0.7801 & 0.0511 & 17.1042 & 0.4012 & 26.5438 & 0.1356\\
    Ensembles & \textbf{-2.6581}  & \textbf{0.0304}  & 0.0892 & 0.0417 & \textbf{-2.3884} & \textbf{0.0593} & 0.5996 & 0.0425 & 4.5459 & \textbf{0.3440} & 22.6413 & 0.1323\\
    \textbf{Ours} & -2.1203 & 0.0362  & \textbf{0.0810} & \textbf{0.0377} & -2.2337 & 0.0601 & \textbf{0.5431} & \textbf{0.0445} & \textbf{4.4738}  & 0.3596 & \textbf{20.7369}  & \textbf{0.1359} \\
    \bottomrule
    \end{tabular}}
    \label{tab:depth-estimation}
\end{table*}

\subsection{Toy Dataset}
We first qualitatively compare the performance of our approach against a set of baselines on a one-dimensional cubic regression dataset. The detailed baselines, implementation, and source code are in Apd.~\ref{apd:baselines} and Apd.~\ref{apd:implementation}. Following~\citet{lobato2015probabilistic}, we train models on $x$ within $\pm4$ and test within $\pm7$. We compare uncertainty estimation for baseline methods. From Fig.~\ref{fig:demo}, we observe that, as expected, all methods accurately capture uncertainty within the \acrshort{IID} training distribution. Regarding \acrshort{OOD} test set, our proposed Density-Regression estimates uncertainty appropriately and the confidence interval grows on \acrshort{OOD} data, without dependence on sampling.

\subsection{Time Series Weather Forecasting}
\begin{figure}[ht!]
\begin{center}
  \includegraphics[width=1.0\linewidth]{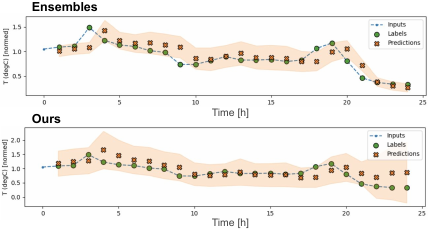}
\end{center}
\caption{Comparison between Deep Ensembles and our model regarding temperature in Celsius (normalized) for every hour on the same day. More details are in Apd.~\ref{apd:time_series}.} 
\label{fig:time-series}
\end{figure}

In this setting, we evaluate models in the real world with a weather time series dataset recorded by the Max Planck Institute for Biogeochemistry~\citep{tensorflow2015-whitepaper}. Specifically, the model will learn to predict the temperature for every hour from 14 different features (collected every 10 minutes), e.g., air temperature, atmospheric pressure, humidity, etc.

Tab.~\ref{tab:time-series} quantitatively shows our method outperforms others in terms of calibration and sharpness in both \acrshort{IID} and \acrshort{OOD} settings. Notably, its predictive uncertainty is more calibrated and sharper than the well-known \acrshort{SOTA} Ensembles. To take a closer look at this difference, we visualize the temperature prediction on \acrshort{OOD} data in Fig.~\ref{fig:time-series}. We observe that Deep Ensembles is conservative, even when making incorrect predictions. In contrast, our model enables signaling when it is likely wrong. Importantly, it is more calibrated by always covering the true label in our confidence intervals, while at the same time, preserving the sharpness with correct predictions.

\subsection{Benchmark UCI}
We next conduct another benchmarking experiment on the UCI datasets~\citep{amini2020deep}. For testing on \acrshort{IID} data, we leave tables in Apd.~\ref{apd:uci}, where our main observation is our Density-Regresiion still performs well in these \acrshort{IID}-only dataset by having a low calibration and sharpness value. We mainly discuss the main result under distribution shifts on UCI:~Wine Quality, which includes two datasets, related to red and white vinho verde wine samples, from the north of Portugal. The goal is to model wine quality based on physicochemical tests~\citep{Cortez2009ModelingWP}. In our setting, we train models on red wine samples, and test on both the \acrshort{IID} samples and the white vinho verde wine \acrshort{OOD} samples. Tab.~\ref{tab:uci} shows our Density-Regression outperforms other baselines and has a competitive result with the \acrshort{SOTA} Deep Enselbmes in all criteria. Notably, the competitive result is not only in uncertainty quality but also in \acrshort{NLL} and \acrshort{RMSE}, showing our method is robust under distribution shifts.

\subsection{Monocular Depth Estimation}
\begin{figure*}[t!]
    \centering
    \begin{tabular}{ccc}
         \includegraphics[width=0.33\textwidth]{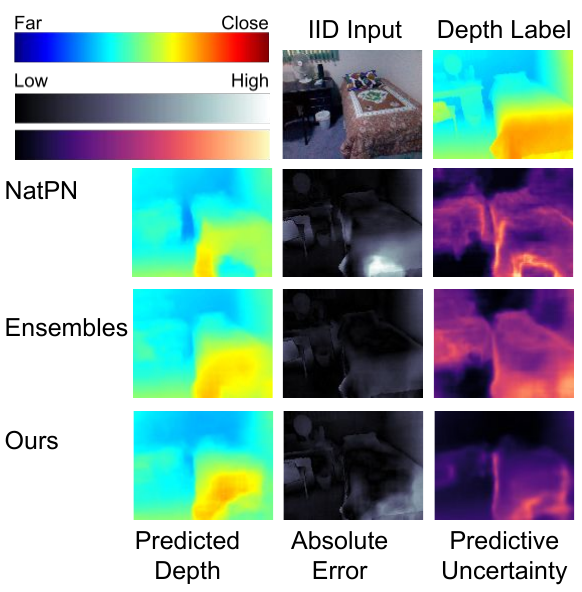}&
         \includegraphics[height=5.65cm,width=0.33\textwidth,keepaspectratio]{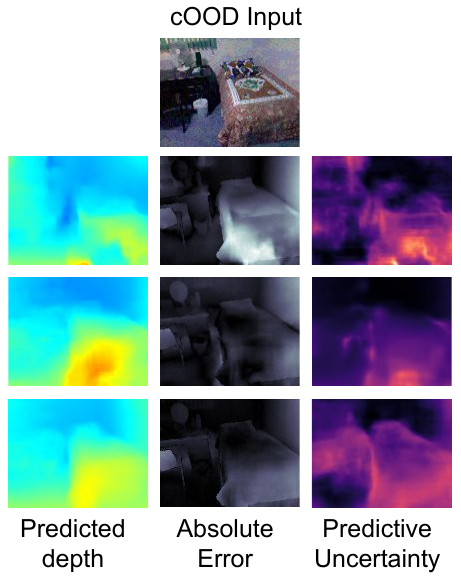}&
         \includegraphics[height=5.29cm,width=0.33\textwidth,keepaspectratio]{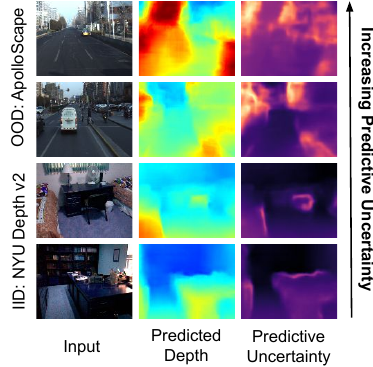}\\
         (a) & (b) & (c)
    \end{tabular}
    \caption{Comparison in pixel-wise depth predictions and predictive uncertainty on (a) \acrshort{IID} and (b) $0.04$ noise level on corrupted \acrshort{OOD} dataset. Detailed figures for the robustness under corrupted noise are in Apd.~\ref{apd:depth}; (c) Our model performance on the real-world \acrshort{OOD} ApolloScape.}
    \label{fig:depth_result}
\end{figure*}

Finally, we present an application of our model in the monocular depth estimation task~\citep{amini2020deep}. We train models with U-Net~\citep{ronneberger2015UNet} on $27k$ RGB-to-depth image pairs of indoor scenes (e.g., bedrooms, kitchens, etc.) from the NYU~Depth~v2~\citep{Silberman2012Indoor}. Then, we test on this \acrshort{IID} test-set and two \acrshort{OOD} datasets, including its corrupted data by adding noise with the Fast Gradient Sign Method~\citep{ian2015explaining} and a diverse outdoor driving ApolloScape~\citep{huang2018ApolloScape}. Tab.~\ref{tab:depth-estimation} once again confirms Density-Regression consistently provides a good quality uncertainty estimate by always achieving the lowest calibration error across three test sets, especially on the real-world \acrshort{OOD} outdoor scenes dataset. Regarding the sharpness, it is worth noticing that, unlike calibration, sharpness is neither a sufficient nor necessary condition for a good uncertainty estimate. It will depend on the model’s accuracy: if a model has a high \acrshort{RMSE} (e.g., on \acrshort{OOD} outdoor scenes), the sharpness score should be high. Therefore, we marked the bold by comparing the calibration error first, then the sharpness later, i.e., if two models have the same calibration error, we compare their sharpness to select the better one.

To take a closer look at the performance, we visualize Fig.~\ref{fig:depth_result}~(a) and Fig.~\ref{fig:depth_result}~(b) to compare the predicted depth map, its error comparing to the ground truth, and the corresponding predicted uncertainty for every pixel. Firstly, we observe that Density-Regression is more robust (i.e., the robustness performance on the \acrshort{OOD} data) than \acrshort{NatPN}~\citep{charpentier2022natural} by lower pixel values in the absolute error image. Secondly, it provides more certain predictions for correct pixels and less certain for incorrect ones than other methods. Finally, it can provide certain predictions on \acrshort{IID} data, and when the corrupted \acrshort{OOD} images are far from the training set, its certainty decreases correspondingly. This confirms a hypothesis that our model is distance-aware, helping to improve the quality of uncertainty quantification under distribution shifts.

Regarding the real-world \acrshort{OOD} ApolloScape, we also observe from Fig.~\ref{fig:depth_result}~(c) that our model can provide confident prediction on \acrshort{IID} images. And when the \acrshort{OOD} dataset is far from \acrshort{IID}, its predictive uncertainty also increases correspondingly. Additionally, Fig.~\ref{fig:model_size}~(c) also shows that when compared to other methods under distribution shifts, our model is also more well-calibrated by less over-confidence and under-confidence, resulting in the closest to the ideal calibration line.

\subsection{Test-time Efficiency Evaluation}
So far, we have empirically shown our model can provide a high-quality uncertainty estimation and have a competitive result with the \acrshort{SOTA} Deep Ensembles across several tasks. We next show that our method outperforms Deep Ensembles regarding test-time efficiency. Indeed, Fig.~\ref{fig:model_size}~(a) and Fig.~\ref{fig:model_size}~(b) show our model has less than four times the model size and five times faster computation than Ensembles across different GPU architectures. Notably, when compared to the best efficient model, i.e., Deterministic~Gaussian~\acrshort{DNN}, our model is only slightly higher due to $p(Z;\alpha)$ density function (the minus of ours to Deterministic can measure this gap). As a result, our method not only has a better uncertainty quality, but also has a higher test-time efficiency than many sampling-free methods, e.g., Quantile, \acrshort{EDL}, \acrshort{NatPN}, \acrshort{SNGP}, and \acrshort{DUE}.

\begin{figure*}[ht!]
    \centering
    \begin{tabular}{ccc}
    \centering
    \includegraphics[width=0.31\linewidth]{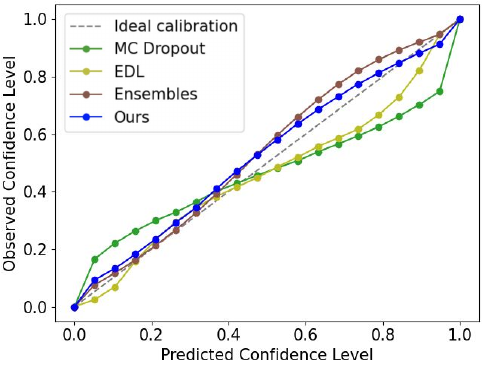}&
    \includegraphics[width=0.31\linewidth]{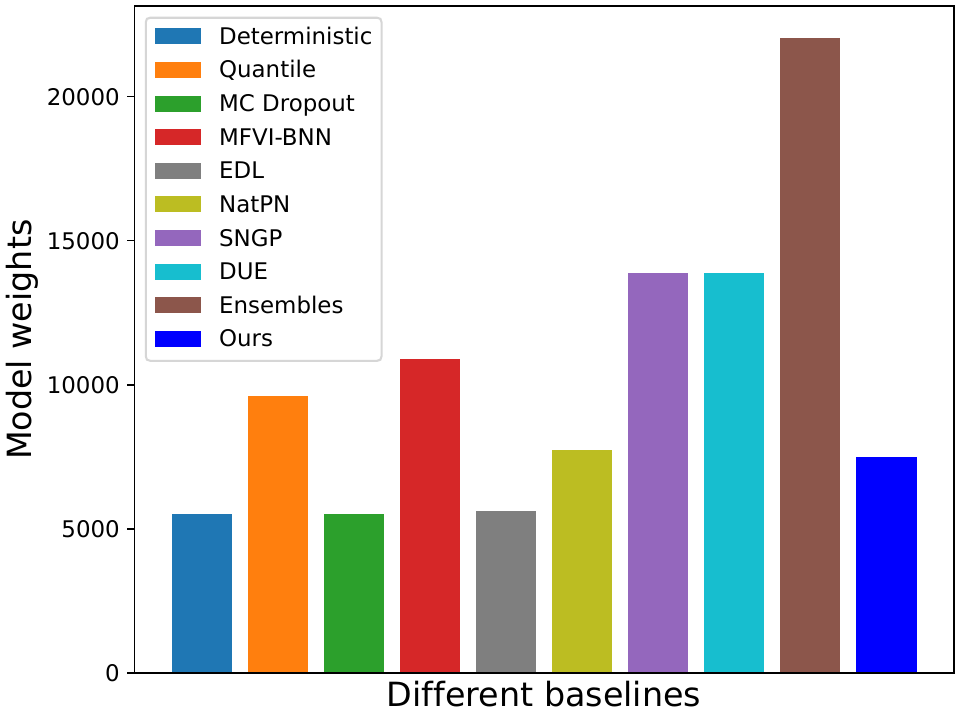}&
    \includegraphics[width=0.31\linewidth]{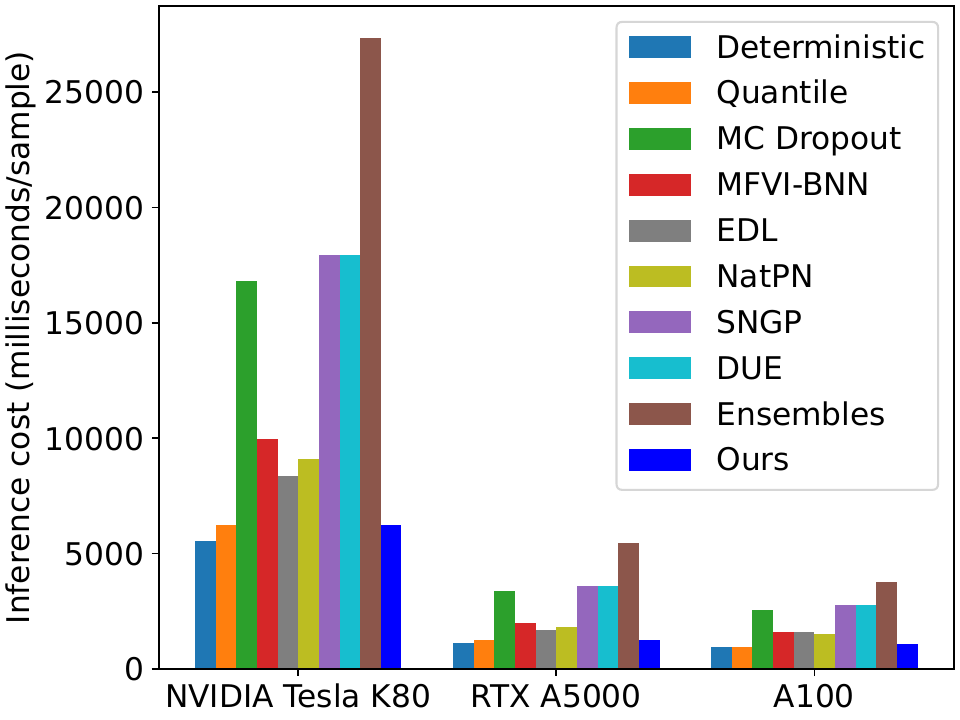}\\
        (a) & (b) & (c)
    \end{tabular}
    \caption{(a) Visualization of calibration error with reliability diagram on the real-world \acrshort{OOD} ApolloScape;
    (b) Comparison in model storage requirement at test-time; (c) Inference cost comparison at test-time across three modern GPU architectures (detailed in Apd.~\ref{apd:implementation}).}
    \label{fig:model_size}
\end{figure*}
\section{Related work}
\textbf{Uncertainty and Robustness.} More discussion of Uncertainty and Robustness can be found in the literature of ~\citet{tran2022plex,bui2022benchmark,ovadia2019can,koh2021wilds,minderer2021revisiting,hendrycks2018benchmarking}. For the regression setting, \citet{gustafsson2023reliable,tran2020methods,dheur2023largescale} have recently empirically covered \acrshort{SOTA} techniques. To summarize, there are three main approaches: sampling-based, sampling-free, and replacing loss function. Regarding \textbf{sampling-based models}, the typical approaches includes \acrshort{GP}~\citep{gardner2018GPyTorch,lee2018deepGP}, \acrshort{BNN}~\citep{blundell2015weight, wen2018flipout}, Dropout~\citep{gal2016mcdropout,gal2016dropouttheory,gal2017concretedropout}, and the \acrshort{SOTA} Ensembles~\citep{hansen1990ensembles}. These sampling-based models, however, are struggling with scalability in terms of the number of weights and inference speed~\citep{nado2021uncertainty}. To tackle this challenge, there are some lightweight sampling-based models have been proposed recently, including BatchEnsemble~\citep{Wen2020BatchEnsemble}, Rank-1~BNN~\citep{dusenberry2020rank1}, and Heteroscedastic~\citep{collier2021correlated}. However, these works often only focus on the classification setting.

Therefore, novel \textbf{sampling-free approaches} to improve uncertainty estimation have recently been proposed. Starting with Deterministic~Gaussian~\acrshort{DNN}~\citep{chua2018deepRL}, then with Quantile Regression~\citep{koenker1978quantiles,romano2019conformalized}, deterministic \acrshort{GP}~\citep{Liu2020SNGP, vanamersfoort2022feature}, and Bayesian inference closed-form based~\citep{amini2020deep,charpentier2020postnet,charpentier2022natural}. However, the uncertainty quality of these methods is often still degraded when compared to sampling-based models. For instance, if the prior hyper-parameters are poorly defined, the Bayesian-based model \acrshort{EDL}~\citep{amini2020deep} and Posterior Network~\citep{charpentier2020postnet} will result in a bad performance in practice. Conversely, Density-Regression does not require any prior hyper-parameters in training and test time. Regarding replacing the loss function, there exists the post-hoc calibration~\citep{kuleshov2018accurate,song19distribution,romano2019conformalized,vovk2005Algorithmic,tibshirani2019Conformal,prinster2022jaws,prinster23jawsx} and regularization approaches~\citep{chung2021beyond,mukhoti2020Calibrating,dheur2023largescale}. In contrast, our baselines only do \acrshort{MLE} using the objective function in Eq.~\ref{eq:MLE} without depending on any regularization or re-calibration set.

\textbf{Improving uncertainty quality via density estimation.} The idea of improving \acrshort{DNN} uncertainty via Density Estimation has been significantly studied~\citep{kuleshov2022sharpness,charpentier2022natural,charpentier2020postnet,mukhoti2022deep,kotelevskii2022NUQ}. However, these work either entail post-hoc re-calibration, \acrshort{OOD} samples in training, or pre-defined prior parameters in test time. For instance, \acrshort{NatPN}~\citep{charpentier2022natural,charpentier2020postnet} customizes the last layer of \acrshort{DNN} with sensitive priors, not normalized by a natural exponent function, resulting in a bad accuracy performance in practice~\citep{nado2021uncertainty}. 

To summarize, compared to prior work, our method does not use any post-hoc re-calibration, \acrshort{OOD} samples in training, or pre-defined prior parameters in test time. Closest to our work is Density-Softmax~\citep{bui2023densitysoftmax}, which is also based on density function and distance awareness property to enhance the quality of uncertainty estimation and robustness under distribution shifts. However, it only works for the classification task, and extending to the regression task is non-trivial. 

\section{Conclusion}
Improving uncertainty quantification is a critical problem in trustworthy~\acrshort{AI}. There has been growing interested in using sampling-based methods to ensure \acrshort{AI} systems are reliable. Challenges often arise when deploying such models in real-world domains. In this regard, our Density-Regression significantly improves test-time efficiency while preserving reliability. We provide theoretical guarantees for our framework and prove it is distance-aware on the feature space. With this property, we empirically show our method is fast, lightweight, and provides high-quality uncertainty estimates in plenty of high-stake real-world applications like weather forecasting and depth estimation. Given this, we hope Density-Regression will inspire researchers and developers to make further progress in improving the efficiency and trustworthiness of \acrshort{DNN}. 

\section*{Acknowledgments}
This work is partially supported by the Discovery Award of the Johns Hopkins University. We thank Drew Prinster (Johns Hopkins University) for his helpful revision of our final paper version.
\bibliographystyle{unsrtnat}
\bibliography{refs}
\printglossary[type=\acronymtype,nonumberlist]
\printglossary
\section*{Checklist}
 \begin{enumerate}

 \item For all models and algorithms presented, check if you include:
 \begin{enumerate}
   \item A clear description of the mathematical setting, assumptions, algorithm, and/or model. [Yes]
   \item An analysis of the properties and complexity (time, space, sample size) of any algorithm. [Yes]
   \item (Optional) Anonymized source code, with specification of all dependencies, including external libraries. [Yes]
 \end{enumerate}

 \item For any theoretical claim, check if you include:
 \begin{enumerate}
   \item Statements of the full set of assumptions of all theoretical results. [Not Applicable]
   \item Complete proofs of all theoretical results. [Yes]
   \item Clear explanations of any assumptions. [Not Applicable]     
 \end{enumerate}

 \item For all figures and tables that present empirical results, check if you include:
 \begin{enumerate}
   \item The code, data, and instructions needed to reproduce the main experimental results (either in the supplemental material or as a URL). [Yes]
   \item All the training details (e.g., data splits, hyperparameters, how they were chosen). [Yes]
         \item A clear definition of the specific measure or statistics and error bars (e.g., with respect to the random seed after running experiments multiple times). [Yes]
         \item A description of the computing infrastructure used. (e.g., type of GPUs, internal cluster, or cloud provider). [Yes]
 \end{enumerate}

 \item If you are using existing assets (e.g., code, data, models) or curating/releasing new assets, check if you include:
 \begin{enumerate}
   \item Citations of the creator If your work uses existing assets. [Yes]
   \item The license information of the assets, if applicable. [Not Applicable]
   \item New assets either in the supplemental material or as a URL, if applicable. [Not Applicable]
   \item Information about consent from data providers/curators. [Not Applicable]
   \item Discussion of sensible content if applicable, e.g., personally identifiable information or offensive content. [Not Applicable]
 \end{enumerate}

 \item If you used crowdsourcing or conducted research with human subjects, check if you include:
 \begin{enumerate}
   \item The full text of instructions given to participants and screenshots. [Not Applicable]
   \item Descriptions of potential participant risks, with links to Institutional Review Board (IRB) approvals if applicable. [Not Applicable]
   \item The estimated hourly wage paid to participants and the total amount spent on participant compensation. [Not Applicable]
 \end{enumerate}

 \end{enumerate}

\appendix
\onecolumn
\noindent\rule{\textwidth}{3pt}
\section*{\Large\centering{Density-Regression: Efficient and Distance-Aware Deep Regressor for Uncertainty Estimation under Distribution Shifts\\(Supplementary Material)}}
\noindent\rule{\textwidth}{1pt}
\textbf{Broader impacts.} High-quality uncertainty estimate is an important property in trustworthy \acrshort{AI}, and Density-Regression can significantly improve uncertainty quality and test-time efficiency. This could be particularly beneficial in high-stake applications (e.g., healthcare, finance, policy decision-making, etc.), where the trained model needs to be deployed and inference on low-resource hardware or real-time response software.

\textbf{Limitations}: 
\begin{enumerate}[leftmargin=13pt,topsep=0pt,itemsep=0mm]
    \item \textbf{Density model performance in practice.} The uncertainty quality of Density-Regression depends on the density function. Our results show if the likelihood on test \acrshort{OOD} feature is lower than \acrshort{IID} set, then Density-Regression can reduce the over-confidence of Deterministic~Gaussian~\acrshort{DNN}. That said, it can be a risk that our model might not fully capture the real-world complexity by estimating density function is not always trivial in practice~\citep{nalisnick2019doDGM,bui2023densitysoftmax,charpentier2022natural}. 
    \item \textbf{Training cost.} Despite showing success in test-time efficiency, we raise awareness about the challenge of training Density-Regression. Specifically, Density-Regression requires a longer training time than Deterministic~Gaussian~\acrshort{DNN} due to three separate training steps in Alg.~\ref{alg:algorithm} (e.g., Fig.~\ref{fig:training-time}). 
\end{enumerate}

\textbf{Remediation.} Given the aforementioned limitations, we encourage people who extend our work to: (1) proactively confront the model design and parameters to desired behaviors in real-world use cases; (2) be aware of the training challenge and prepare enough time to pre-train our framework in practice.

\textbf{Future work.} We plan to tackle Density-Regression's limitations, including improving estimation techniques to enhance the quality of the density function and continuing to reduce the number of parameters to deploy this framework in real-world systems.  

\textbf{Reproducibility.} The source code to reproduce our results is in the attached zip file of this supplementary material. We provide all proofs in Apd.~\ref{apd:proof}, experimental settings in Apd.~\ref{apd:settings}, and detailed results in Apd.~\ref{apd:results}.
\section{Proofs}\label{apd:proof}
In this appendix, we provide the proofs for all the results in the main paper.

\subsection{Lemma~\ref{lemma:exp} and the proof}
\begin{lemma}\label{lemma:exp}
    When $\mathcal{Y}$ is the label space and $\Phi(x,y)$ is the sufficient statistic of the joint distribution $\mathbb{P}(x,y)$ associated with $(x,y)\in \mathcal{X}\times \mathcal{Y}$, we have the predictive distribution follows
    \begin{align}
        p(y|x;\theta) = \exp\left(\eta(\theta_g)^\top\Phi(f(x),y)-\log\left(\int_{y'\in \mathcal{Y}}\exp\left(\eta(\theta_g)^\top\Phi(f(x),y')\right)dy'\right)\right),
    \end{align}
    where $\theta = (\theta_g,\theta_f)$ is the parameter vector of the forecast $h=g\circ f$ and $\eta$ is the natural function for the parameter $\theta$.
\end{lemma}
\begin{proof}\label{proof:lemma:exp}
    We have a random variable $x$ that belongs to the exponential family with the Probability density function
    \begin{align}
        p(x;\theta)=h(x)\exp(\eta(\theta)^\top\Phi(x)-A(\theta)),
    \end{align}
    where $h(x)$ is the underlying measure, $\theta$ is the parameter vector, $\eta$ is the natural function for parameter $\theta$, $\Phi(x)$ is the sufficient statistic, and the log-normalizer 
    \begin{align}
        A(\theta) = \log \int_\mathcal{X} h(x)\exp\left(\eta(\theta)^\top\Phi(x)\right)dx.
    \end{align}
    For $\mathcal{Y}$ is the label space and $\Phi(x,y)$ is the sufficient statistic of the joint distribution $\mathbb{P}(x,y)$ associated with $(x,y)\in \mathcal{X}\times \mathcal{Y}$, we have the join Probability density function
    \begin{align}
      p(x,y;\theta)=h(x)\exp\left(\eta(\theta)^\top\Phi(x,y)-A(\theta)\right).
    \end{align}
    By Bayes's rule, we obtain
    \begin{align}
        p(y|x;\theta) = \frac{p(x|y;\theta)}{\int_{y'\in \mathcal{Y}}p(x|y';\theta)p(y'|\theta)dy'}=\frac{p(x,y;\theta)}{\int_{y'\in \mathcal{Y}}p(x,y';\theta)dy'}=\frac{h(x)\exp\left(\eta(\theta)^\top\Phi(x,y)-A(\theta)\right)}{\int_{y'\in \mathcal{Y}}h(x)\exp\left(\eta(\theta)^\top\Phi(x,y')-A(\theta)\right)dy'}.
    \end{align}
    Assumed the $h(x)$ base reference measure to be a function only of $x$ so that we can cancel it from the numerator and denominator in the last step above. Therefore, we get
    \begin{align}
        p(y|x;\theta) = \frac{\exp\left(\eta(\theta)^\top\Phi(x,y)\right)}{\int_{y'\in \mathcal{Y}}\exp\left(\eta(\theta)^\top\Phi(x,y')\right)dy'}=\exp\left(\eta(\theta)^\top\Phi(x,y)-\log\left(\int_{y'\in \mathcal{Y}}\exp\left(\eta(\theta)^\top\Phi(x,y')\right)dy'\right)\right).
    \end{align}
    As a result, when $y = g(f(x))$, we obtain
    \begin{align}
        p(y|x;\theta) &=\exp\left(\eta(\theta_g)^\top\Phi(f(x),y)-\log\left(\int_{y'\in \mathcal{Y}}\exp\left(\eta(\theta_g)^\top\Phi(f(x),y')\right)dy'\right)\right)\\
        &= \frac{\exp\left(\eta(\theta_g)^\top\Phi(f(x),y)\right)}{\int_{y'\in \mathcal{Y}}\exp\left(\eta(\theta_g)^\top\Phi(f(x),y')\right)dy'}
    \end{align}
    of Lemma~\ref{lemma:exp}
\end{proof}

\subsection{Proof of Theorem~\ref{theo:gauss}}\label{proof:theo:gauss}
\begin{proof}
    Since the sufficient statistic has the form $\Phi(z,y) = \begin{bmatrix}zy^2 & y^2 & 2zy & 2y & z & 1\end{bmatrix}$, where $z=f(x)$, replace it to the dot product of the natural parameter $\eta(\theta)$ with sufficient statistic $\Phi(z,y)$, we get
    \begin{align}
        \eta(\theta_g)^\top \Phi(z,y) &= \begin{bmatrix}\theta_1&\theta_2&\theta_3&\theta_4&\theta_5&\theta_6\end{bmatrix}\begin{bmatrix}zy^2 & y^2 & 2zy & 2y & z & 1\end{bmatrix}^\top\\
        &=\begin{bmatrix}\theta_1&\theta_2\end{bmatrix}\begin{bmatrix}z\\1\end{bmatrix}y^2 + 2\begin{bmatrix}\theta_3&\theta_4\end{bmatrix}\begin{bmatrix}z\\1\end{bmatrix}y + \begin{bmatrix}\theta_5&\theta_6\end{bmatrix}\begin{bmatrix}z\\1\end{bmatrix}\\       &=\begin{bmatrix}\theta_1&\theta_2\end{bmatrix}\begin{bmatrix}z\\1\end{bmatrix}y^2 + 2\begin{bmatrix}\theta_3&\theta_4\end{bmatrix}\begin{bmatrix}z\\1\end{bmatrix}y + \frac{\left(-\begin{bmatrix}\theta_3&\theta_4\end{bmatrix}\begin{bmatrix}z\\1\end{bmatrix}\right)^2}{\begin{bmatrix}\theta_1&\theta_2\end{bmatrix}\begin{bmatrix}z\\1\end{bmatrix}}\\
        &=\theta_g^\sigma\begin{bmatrix}z\\1\end{bmatrix}y^2 + 2\theta_{(y,x_1)}\begin{bmatrix}z\\1\end{bmatrix}y+\frac{\left(-\theta_g^\mu\begin{bmatrix}z\\1\end{bmatrix}\right)^2}{\theta_g^\sigma\begin{bmatrix}z\\1\end{bmatrix}}.
    \end{align}
    Hence, we get
    \begin{align}\label{eq:proof:theo:gauss:exp_in}
        -p(z;\alpha)\eta(\theta_g)^\top \Phi(z,y) = -p(z;\alpha)\theta_g^\sigma\begin{bmatrix}z\\1\end{bmatrix}y^2 - 2p(z;\alpha)\theta_g^\mu\begin{bmatrix}z\\1\end{bmatrix}y-p(z;\alpha)\frac{\left(-\theta_g^\mu\begin{bmatrix}z\\1\end{bmatrix}\right)^2}{\theta_g^\sigma\begin{bmatrix}z\\1\end{bmatrix}},
    \end{align}
    and the corresponding log-normalizer
    \begin{align}\label{eq:proof:theo:gauss:log_in}
        &\log \int_{\mathcal{Y}}\exp\left(-p(z;\alpha)\eta(\theta_g)^\top \Phi(z,y)\right)dy\\ &= \log \int_{\mathcal{Y}} \exp\left(-p(z;\alpha)\theta_g^\sigma\begin{bmatrix}z\\1\end{bmatrix}y^2 - 2p(z;\alpha)\theta_g^\mu\begin{bmatrix}z\\1\end{bmatrix}y-p(z;\alpha)\frac{\left(-\theta_g^\mu\begin{bmatrix}z\\1\end{bmatrix}\right)^2}{\theta_g^\sigma\begin{bmatrix}z\\1\end{bmatrix}}\right)dy\\
        &= \log\left(\sqrt{\frac{\pi}{p(z;\alpha)\theta_g^\sigma\begin{bmatrix}z\\1\end{bmatrix}}}\exp\left(\underbrace{\frac{\left(2p(z;\alpha)\theta_g^\mu\begin{bmatrix}z\\1\end{bmatrix}\right)^2}{4p(z;\alpha)\theta_g^\sigma\begin{bmatrix}z\\1\end{bmatrix}}-p(z;\alpha)\frac{\left(-\theta_g^\mu\begin{bmatrix}z\\1\end{bmatrix}\right)^2}{\theta_g^\sigma\begin{bmatrix}z\\1\end{bmatrix}}}_{0}\right)\right)\\
        &= \log\left(\sqrt{\frac{\pi}{p(z;\alpha)\theta_g^\sigma\begin{bmatrix}z\\1\end{bmatrix}}}\right) = \frac{1}{2}\log\left(\frac{\pi}{p(z;\alpha)\theta_g^\sigma\begin{bmatrix}z\\1\end{bmatrix}}\right).
    \end{align}
    So, replace the result of Equation~\ref{eq:proof:theo:gauss:exp_in} and Equation~\ref{eq:proof:theo:gauss:log_in} to the Exponential form of Density-Regressor, we obtain
    \begin{align}\label{eq:proof:theo:gauss:predictive}
        &p(y|x;\theta) = \exp\left(-p(z;\alpha)\eta(\theta_g)^\top \Phi(z,y) - \log \int_{\mathcal{Y}}\exp\left(-p(z;\alpha)\eta(\theta_g)^\top \Phi(z,y)\right)dy\right) \text{, where }z =f(x)\\
        &= \exp\left(-p(z;\alpha)\theta_g^\sigma\begin{bmatrix}z\\1\end{bmatrix}y^2 - 2p(z;\alpha)\theta_g^\mu\begin{bmatrix}z\\1\end{bmatrix}y-p(z;\alpha)\frac{\left(-\theta_g^\mu\begin{bmatrix}z\\1\end{bmatrix}\right)^2}{\theta_g^\sigma\begin{bmatrix}z\\1\end{bmatrix}} - \frac{1}{2}\log\left(\frac{\pi}{p(z;\alpha)\theta_g^\sigma\begin{bmatrix}z\\1\end{bmatrix}}\right)\right)\\
        &= \exp\{\frac{-1}{2}2\cdot p(z;\alpha)\theta_g^\sigma\begin{bmatrix}z\\1\end{bmatrix}y^2 + 2\cdot p(z;\alpha)\theta_g^\sigma\begin{bmatrix}z\\1\end{bmatrix}\frac{-\theta_g^\mu\begin{bmatrix}z\\1\end{bmatrix}}{\theta_g^\sigma\begin{bmatrix}z\\1\end{bmatrix}}y\\&-\left(\frac{-\theta_g^\mu\begin{bmatrix}z\\1\end{bmatrix}}{\theta_g^\sigma\begin{bmatrix}z\\1\end{bmatrix}}\right)^2\left(\frac{1}{2}2\cdot p(z;\alpha)\theta_g^\sigma\begin{bmatrix}z\\1\end{bmatrix}\right) - \frac{1}{2}\log\left(2\pi\frac{1}{2\cdot p(z;\alpha)\theta_g^\sigma\begin{bmatrix}z\\1\end{bmatrix}}\right)\}.
    \end{align}
    On the other hand, we have the Probability density function of the Gaussian distribution $\mathcal{N}(\mu,\sigma^2)$ has the form
    \begin{align}\label{eq:proof:theo:gauss:tmp_predictive}
        p(y|x;\theta) &= \exp\left(\frac{-1}{2\sigma^2}y^2 + \frac{\mu}{\sigma^2}y -\frac{1}{2\sigma^2}\mu^2 - \frac{1}{2}\log(2\pi\sigma^2)\right)\\
        &= \exp\left(\frac{-1}{2\sigma^2}y^2 + \frac{\mu}{\sigma^2}y - \left(\frac{1}{2\sigma^2}\mu^2 + \frac{1}{2}\log(2\pi\sigma^2)\right)\right).
    \end{align}
    Applying the result from Lemma~\ref{lemma:exp}, when the Exponential family has the form
    \begin{align}
        p(y|x;\theta) = \exp\left(\eta(\theta_g)^\top \Phi(z=f(x),y) - \log \int_{\mathcal{Y}}\exp\left(\eta(\theta_g)^\top \Phi(z=f(x),y)\right)dy\right),
    \end{align}
    we obtain
    \begin{align}\label{eq:proof:theo:gauss:log_tmp_predictive}
        \log \int_\mathcal{Y} \exp\left(\frac{-1}{2\sigma^2}y^2 + \frac{\mu}{\sigma^2}y\right)dy = \log\left(\sqrt{\pi2\sigma^2}\exp\left(\frac{\mu^2}{\sigma^4} \frac{2\sigma^2}{4}\right)\right) = \frac{1}{2\sigma^2}\mu^2 + \frac{1}{2}\log(2\pi\sigma^2).
    \end{align}
    Combining the result from Equation~\ref{eq:proof:theo:gauss:predictive}, Equation~\ref{eq:proof:theo:gauss:tmp_predictive}, and Equation~\ref{eq:proof:theo:gauss:log_tmp_predictive}, we get
    \begin{align}
        p(y|x;\theta) &= \exp\{\frac{-1}{2}\underbrace{2\cdot p(z;\alpha)\theta_g^\sigma\begin{bmatrix}z\\1\end{bmatrix}}_{\sigma^2(x,\theta)^{-1}}y^2 + \underbrace{2\cdot p(z;\alpha)\theta_g^\sigma\begin{bmatrix}z\\1\end{bmatrix}}_{\sigma^2(x,\theta)^{-1}}\underbrace{\frac{-\theta_g^\mu\begin{bmatrix}z\\1\end{bmatrix}}{\theta_g^\sigma\begin{bmatrix}z\\1\end{bmatrix}}}_{\mu(x,\theta)}y\\&-\left(\underbrace{\left(\frac{-\theta_g^\mu\begin{bmatrix}z\\1\end{bmatrix}}{\theta_g^\sigma\begin{bmatrix}z\\1\end{bmatrix}}\right)^2}_{\mu(x,\theta)^2}\underbrace{\left(\frac{1}{2}2\cdot p(z;\alpha)\theta_g^\sigma\begin{bmatrix}z\\1\end{bmatrix}\right)}_{\left(2\sigma^2(x.\theta)\right)^{-1}} + \frac{1}{2}\log\left(2\pi\underbrace{\frac{1}{2\cdot p(z;\alpha)\theta_g^\sigma\begin{bmatrix}z\\1\end{bmatrix}}}_{\sigma^2(x,\theta)}\right)\right)\}.
    \end{align}
    As a result, we obtain $p(y|x;\theta) \sim \mathcal{N}(\mu(x,\theta),\sigma^2(x,\theta))$, where
    \begin{align}
        \mu(x,\theta) = -\left(\theta_g^\sigma\begin{bmatrix}
            z\\1
        \end{bmatrix}\right)^{-1}\left(\theta_g^\mu{\begin{bmatrix}
            z\\1
        \end{bmatrix}}\right) \text{ and }
        \sigma^2(x,\theta) = \left(2\cdot p(z;\alpha)\theta_g^\sigma\begin{bmatrix}
            z\\1
        \end{bmatrix}\right)^{-1}
    \end{align} of Theorem~\ref{theo:gauss}.
\end{proof}

\subsection{Proof of Corollary~\ref{corol:gauss}}\label{proof:corol:gauss}
\begin{proof}
From the result of Theorem~\ref{theo:gauss}, we have
\begin{align}
        \mu(x,\theta) = -\left(\theta_g^\sigma\begin{bmatrix}
            z\\1
        \end{bmatrix}\right)^{-1}\left(\theta_g^\mu{\begin{bmatrix}
            z\\1
        \end{bmatrix}}\right) \text{ and }
        \sigma^2(x,\theta) = \left(2\cdot p(z;\alpha)\theta_g^\sigma\begin{bmatrix}
            z\\1
        \end{bmatrix}\right)^{-1}.
\end{align}
Hence, firstly, we can rewrite the mean by
\begin{align}
    \mu(x,\theta) = \underbrace{\left(2\cdot p(z;\alpha)\theta_g^\sigma\begin{bmatrix}z\\1\end{bmatrix}\right)^{-1}}_{\sigma^2(z,\theta)} -2 \cdot p(z;\alpha)\theta_g^\mu{\begin{bmatrix}z\\1\end{bmatrix}}.
\end{align}
On the other hand, we have the log of standard deviation as follows
\begin{align}
    \log(\sigma(x,\theta)) &= \log\left(\frac{1}{\sqrt{2\cdot p(z;\alpha)\theta_g^\sigma\begin{bmatrix}
            z\\1
        \end{bmatrix}}}\right) = -\log\left(2\cdot p(z;\alpha)\theta_g^\sigma\begin{bmatrix}
            z\\1
        \end{bmatrix}\right)^{\frac{1}{2}}\\
    &= -\frac{1}{2}\log\left(2\cdot p(z;\alpha)\theta_g^\sigma\begin{bmatrix}
            z\\1
        \end{bmatrix}\right)
    = -\frac{1}{2}\left[\log(2)+\log(p(z;\alpha))+\log(\theta_g^\sigma\begin{bmatrix}
            z\\1
        \end{bmatrix})\right]\\
    &=-\frac{1}{2}\left[\log(2)+\log(p(z;\alpha))+\theta_g^\sigma\begin{bmatrix}z\\1\end{bmatrix}\right] \text{ (parameterised directly the $\log$ by $\theta_g^\sigma\begin{bmatrix}z\\1\end{bmatrix}$)}.
\end{align}
Therefore, we obtain the variance
\begin{align}
    \sigma^2(x;\theta)=\exp\left(\log(\sigma(x,\theta))\right)^2
    = \exp\left(-\frac{1}{2}\left[\log(2)+\log(p(z;\alpha))+\theta_g^\sigma\begin{bmatrix}z\\1\end{bmatrix}\right]\right)^2
\end{align}
of Corollary~\ref{corol:gauss}.
\end{proof}

\subsection{Proof of Theorem~\ref{theorem:uncertainty_iid_ood}}
\begin{proof}\label{proof:theorem:uncertainty_iid_ood}
    Using Lemma~\ref{lemma:density},  we have $\lim_{d(z_{t},Z_s)\rightarrow \infty}p(z_{t};\alpha)\rightarrow 0$, i.e., if $d(z_{ood},Z_s)\rightarrow \infty$ then $p(z_{ood};\alpha)\rightarrow 0$. Therefore, we obtain
    \begin{align}
        \lim_{p(z_{ood};\alpha)\rightarrow 0} \sigma^2(x_{ood};\theta) = \lim_{p(z_{ood};\alpha)\rightarrow 0} \left(2\cdot p(z_{ood};\alpha)\theta_g^\sigma\begin{bmatrix}
            z\\1
        \end{bmatrix}\right)^{-1} = \infty.
    \end{align}
    Conversely, we have $\lim_{d(z_{t},Z_s)\rightarrow 0}p(z_{t};\alpha)\rightarrow \infty$, i.e., if $d(z_{iid},Z_s)\rightarrow 0$ then $p(z_{iid};\alpha)\rightarrow \infty$. Therefore, we obtain
    \begin{align}
        \lim_{p(z_{iid};\alpha)\rightarrow \infty} \sigma^2(x_{iid};\theta) = \lim_{p(z_{iid};\alpha)\rightarrow \infty} \left(2\cdot p(z_{iid};\alpha)\theta_g^\sigma\begin{bmatrix}
            z\\1
        \end{bmatrix}\right)^{-1} = 0
    \end{align}
    of Theorem~\ref{theorem:uncertainty_iid_ood}. 
\end{proof}

\subsection{Proof of Theorem~\ref{theorem:distance-aware}}
\begin{proof}\label{proof:theorem:distance-aware}
The proofs contain three parts. The first part shows density function $p(z_t;\alpha)$ is monotonically decreasing w.r.t. distance function $\mathbb{E}\left \| z_{t} - Z_s \right \|_{\mathcal{Z}}$. The second part shows the metric $u(x_t)$ is maximized when $p(z_t;\alpha) \rightarrow 0$. The third part shows $u(x_t)$ monotonically decreasing w.r.t. $p(z_t;\alpha)$ on the interval $\left ( 0, \infty\right ]$.

\textit{\underline{Part (1). The monotonic decrease of density function $p(z_t;\alpha)$ w.r.t. distance function $\mathbb{E}\left \| z_{t} - Z_s \right \|_{\mathcal{Z}}$}}: Consider the probability density function $p(z_t;\alpha)$ follows Normalizing-Flows which output the Gaussian distribution with mean (median) $\mu$ and standard deviation $\sigma$, then we have
\begin{align}
p(z_t;\alpha) = \frac{1}{\sigma \sqrt{2\pi}} \exp\left(\frac{-1}{2}\left( \frac{z_t - \mu}{\sigma} \right)^2\right).
\end{align}

Take derivative, we obtain
\begin{align}\label{eq:proof:pdf}
    \frac{d}{d z_t}p(z_t;\alpha) &= \left[ \frac{-1}{2}\left( \frac{z_t - \mu}{\sigma} \right)^2\right]' p(z_t;\alpha)
    = \frac{\mu - z_t}{\sigma^2} p(z_t;\alpha) \Rightarrow
     \begin{cases} 
\frac{d}{d z_t}p(z_t;\alpha) > 0 &\text{ if } z_t < \mu,\\ 
\\
\frac{d}{d z_t}p(z_t;\alpha) = 0 &\text{ if } z_t = \mu,\\
\\
\frac{d}{d z_t}p(z_t;\alpha) < 0 &\text{ if } z_t > \mu.
    \end{cases}
\end{align}

Consider the distance function $\mathbb{E}\left \| z_{t} - Z_s \right \|_{\mathcal{Z}}$ follows the absolute norm, then we have
\begin{align}
    \mathbb{E}\left \| z_{t} - Z_s \right \|_{\mathcal{Z}} = \mathbb{E}\left( \left| z_{t} - Z_s \right| \right) = \int_{-\infty}^{z_t} \mathbb{P}(Z_s \leq t) dt + \int_{z_t}^{+\infty} \mathbb{P}(Z_s \geq t) dt.
\end{align}

Take derivative, we obtain
\begin{align}\label{eq:proof:norm}
    \frac{d}{d z_t}\mathbb{E}\left \| z_{t} - Z_s \right \|_{\mathcal{Z}} = \mathbb{P}(Z_s \leq z_{t}) - \mathbb{P}(Z_s \geq z_{t}) \Rightarrow
     \begin{cases} 
\frac{d}{d z_t}\mathbb{E}\left \| z_{t} - Z_s \right \|_{\mathcal{Z}} < 0 &\text{ if } z_t < \mu,\\ 
\\
\frac{d}{d z_t}\mathbb{E}\left \| z_{t} - Z_s \right \|_{\mathcal{Z}} = 0 &\text{ if } z_t = \mu,\\
\\
\frac{d}{d z_t}\mathbb{E}\left \| z_{t} - Z_s \right \|_{\mathcal{Z}} > 0 &\text{ if } z_t > \mu.
    \end{cases}
\end{align}

Combining the result in Equation~\ref{eq:proof:pdf} and Equation~\ref{eq:proof:norm}, we have $p(z_t;\alpha)$ is maximized when $\mathbb{E}\left \| z_{t} - Z_s \right \|_{\mathcal{Z}}$ is minimized at the median $\mu$, $p(z_t;\alpha)$ increase when $\mathbb{E}\left \| z_{t} - Z_s \right \|_{\mathcal{Z}}$ decrease and vice versa. As a consequence, we obtain $p(z_t;\alpha)$ is monotonically decreasing w.r.t. distance function $\mathbb{E}\left \| z_{t} - Z_s \right \|_{\mathcal{Z}}$.

\textit{\underline{Part (2). The maximum of metric $u(x_t)$}}: Consider $u(x_t)=v(d(x_t,X_s))$ in Def.~\ref{def:distanceaware}, let $u(x_t)$ is the entropy of predictive distribution of Density-Regression 
$p(y|x;\theta)$, i.e., $u(x_t)=\mathrm{H}(Y|X=x_t)$, where $\mathrm{H}(Y|X=x_t)$ be the entropy of $Y$ conditioned on $X$ taking a certain value $x_t$ via $p(y|x;\theta)$. When the predictive distribution of Density-Regression follows the Normal distribution, i.e., $p(y|x;\theta) \sim \mathcal{N}(\mu(x,\theta),\sigma^2(x,\theta))$, we have 
\begin{align}
     u(x_t) &= \mathrm{H}(Y|X=x_t) = -\int_{\mathcal{Y}} p(y|x_t)\log(p(y|x_t))dy\\
     &= -\int_{\mathcal{Y}}\frac{1}{\sqrt{2\pi\sigma^2(x_t,\theta)}}\exp\left(-\frac{\left(y-\mu(x_t,\theta)^2\right)}{2\sigma^2(x_t,\theta)}\right)\log\left[\frac{1}{\sqrt{2\pi\sigma^2(x_t,\theta)}}\exp\left(-\frac{\left(y-\mu(x_t,\theta)^2\right)}{2\sigma^2(x_t,\theta)}\right)\right]dy\\
     &= -\int_{\mathcal{Y}}\frac{1}{\sqrt{2\pi\sigma^2(x_t,\theta)}}\exp\left(-\frac{\left(y-\mu(x_t,\theta)^2\right)}{2\sigma^2(x_t,\theta)}\right)\left[\log\left(\frac{1}{\sqrt{2\pi\sigma^2(x_t,\theta)}}\right) - \left(\frac{\left(y-\mu(x_t,\theta)^2\right)}{2\sigma^2(x_t,\theta)}\right)\right]dy\\
     &= -\log\left(\frac{1}{\sqrt{2\pi\sigma^2(x_t,\theta)}}\right)\underbrace{\int_{\mathcal{Y}}\underbrace{\frac{\exp\left(-\frac{\left(y-\mu(x_t,\theta)^2\right)}{2\sigma^2(x_t,\theta)}\right)}{\sqrt{2\pi\sigma^2(x_t,\theta)}}}_{p(y|x_t)}dy}_{1}+\int_{\mathcal{Y}}\left(\frac{\left(y-\mu(x_t,\theta)^2\right)}{2\sigma^2(x_t,\theta)}\right)\underbrace{\frac{\exp\left(-\frac{\left(y-\mu(x_t,\theta)^2\right)}{2\sigma^2(x_t,\theta)}\right)}{\sqrt{2\pi\sigma^2(x_t,\theta)}}}_{p(y|x_t)}dy\\
     &= \frac{1}{2}\log\left(2\pi\sigma^2(x_t,\theta)\right)+\frac{1}{2\sigma^2(x_t,\theta)}\int_{\mathcal{Y}}\left(y-\mu(x_t,\theta)\right)^2dy.
\end{align}
Combine with the fact that $\int_{\mathcal{Y}}\left(y-\mu(x_t,\theta)\right)^2dy= \sigma^2(x_t,\theta)$, we obtain
\begin{align}
    u(x_t) = \frac{1}{2}\log\left(2\pi\sigma^2(x_t,\theta)\right)+\frac{1}{2}.
\end{align}
Since $u(x_t)$ is now just a continuous function of its variance $\sigma^2(x_t,\theta)$, it is monotone increasing w.r.t. $\sigma^2(x_t,\theta)$ on the interval $(0,\infty]$. Therefore, we get $u(x_t)$ is maximized when $\sigma^2(x_t;\theta)$ is maximized. As a result, when $\sigma^2(x_t,\theta) = \left(2\cdot p(z_t;\alpha)\theta_g^\sigma\begin{bmatrix}z_t\\1\end{bmatrix}\right)^{-1}$, where $z_t = f(x_t)$, we obtain $u(x_t)$ is maximized when $p(z_t;\alpha) \rightarrow 0$, which will happen if $z_t$ is \acrshort{OOD} data (by the result in Thm.~\ref{theorem:uncertainty_iid_ood} and Proof~\ref{proof:theorem:uncertainty_iid_ood}).

\textit{\underline{Part (3). The monotonically decrease of metric $u(x_t)$ on the interval $\left ( 0, \infty\right ]$}}: Consider the function
\begin{align}
    \mathcal{F}(p(z_t;\alpha)) = \frac{1}{2}\log\left(2\pi\left(2\cdot p(z_t;\alpha)\theta_g^\sigma\begin{bmatrix}z_t\\1\end{bmatrix}\right)^{-1}\right)+\frac{1}{2}.
\end{align}

Let $a= p(z_t;\alpha)$, $b = \theta_g^\sigma\begin{bmatrix}z_t\\1\end{bmatrix}$, then
\begin{align}
    \mathcal{F}(a) = \frac{1}{2}\log\left(\frac{2\pi}{2ab}\right)+\frac{1}{2} = \frac{1}{2}\left[\log(\pi) - \log(ab)\right]+\frac{1}{2},
\end{align}
and we need to find $\frac{d}{d a}\mathcal{F}$. Take derivative, we obtain 
\begin{align}
    \frac{d}{d a}\mathcal{F} = \frac{-1}{2a} < 0 \text{ (due to $a \in \left(0, \infty\right]$)},
\end{align} combining with $u(x_t)$ is maximized if $a \rightarrow 0$, we obtain $u(x_t)$ decrease monotonically on the interval $\left(0,\infty \right]$.

Combining the result in \textit{Part (2).} $u(x_t)$ is maximized if $p(z_t;\alpha) \rightarrow 0$ which will happen if $z_t$ is \acrshort{OOD} data, and the result in \textit{Part (3).} $u(x_t)$ is decrease monotonically w.r.t. $p(z_t;\alpha)$ on the interval $\left ( 0, \infty\right ]$, which will happen if $x_t$ is closer to \acrshort{IID} data since the likelihood value $p(z_t;\alpha)$ increases, we obtain the distance awareness of $-p(z;\alpha) \theta_g^\top \Phi(z,y)$. 

Combining the result in \textit{Part (1).} $p(z_t;\alpha)$ is monotonically decreasing w.r.t. distance function $\mathbb{E}\left \| z_{t} - Z_s \right \|_{\mathcal{Z}}$ and the result \textit{distance awareness} of $-p(z;\alpha) \theta_g^\top \Phi(z,y)$, we obtain the conclusion: $p(y|x;\theta) \propto  \exp(-p(f(x);\alpha) \theta_g^\top \Phi(f(x),y))$ is distance aware on feature space $\mathcal{Z}$ of Theorem~\ref{theorem:distance-aware}. 
\end{proof}
\section{Experimental settings}\label{apd:settings}
In this appendix, we summarize the baselines that we compared in our experiments and provide more detail about our implementation as well as the demo code snippet.

\subsection{Baseline details}\label{apd:baselines}
We provide an exhaustive literature review of 10 \acrshort{SOTA} related methods which are used to make comparisons with our model:
\begin{itemize}
    \item \textbf{Deterministic~\acrshort{DNN}~\citep{vapnik1998erm}} corresponds to Deterministic Regression in Section~\ref{sec:background}.
    \item \textbf{Deterministic~Gaussian~\acrshort{DNN}~\citep{chua2018deepRL}} is discussed in Section~\ref{sec:background}.
    \item \textbf{Quantile~Regression}~\citep{romano2019conformalized} makes the forecast $h: \mathcal{X} \rightarrow \mathbb{R}^2$ to output the prediction intervals with the lower quantile $q^{\alpha/2}(x) = \inf\{y\in \mathbb{R}: F_{Y|X}(y|x) > \alpha/2\}$ and upper quantile $q^{1-\alpha/2}(x) = \inf\{y\in \mathbb{R}: F_{Y|X}(y|x) > (1-\alpha/2)\}$, where $F_{Y|X}$ is the conditional CDF. Then, this forecast $h$ will be trained by using the pinball loss~\citep{Steinwart2011Estimating}.
    \item \textbf{\acrshort{MC}~Dropout~\citep{gal2016mcdropout}} includes dropout regularization method in the model. In test-time, it uses \acrshort{MC} sampling by dropout to make different predictions, then obtain the final mean and variance by Equation.~\ref{eq:sampling-based}.
    \item \textbf{\acrshort{MFVI}~\acrshort{BNN}~\citep{wen2018flipout}} uses the \acrshort{BNN} by putting distribution over the weight by mean and variance per each weight. In test-time, it performs prediction by using Equation.~\ref{eq:sampling-based}. Because each weight consists of mean and variance, the total model weights will double as the Deterministic~\acrshort{DNN}.
    \item \textbf{\acrshort{EDL}~\citep{amini2020deep}} is based on Evidential Deep Learning~\citep{sensoy2018evidential} by making use of conjugate prior property in Bayesian Inference to compute posterior distribution in closed-form. This approach is sensitive to hyper-parameters by requiring to selection of Prior’s parameters.
    \item \textbf{\acrshort{NatPN}~\citep{charpentier2022natural}} is the closest to our work by also estimating the density function on the marginal feature space. However, it is based on \acrshort{EDL} so their loss function and the regressor function are different by the Bayesian approach. Due to belonging to the Bayesian perspective like \acrshort{EDL}, it needs to select a "good" Prior distribution, which is often difficult in practice.
    \item \textbf{\acrshort{SNGP}~\citep{Liu2020SNGP}} is a combination of the last \acrshort{GP} layer with Spectral Normalization to the hidden layers. This algorithm is primarily designed for the classification, however, we can extend it to the regression task by making the \acrshort{GP} layer output the mean and variance like Deterministic~Gaussian~\acrshort{DNN}.
    \item \textbf{\acrshort{DUE}~\citep{vanamersfoort2022feature}} is an extension version of \acrshort{SNGP} by constrain Deep Kernel Learning’s feature extractor to approximately preserve distances through a bi-Lipschitz constraint.
    \item \textbf{Deep Ensembles~\citep{lakshminarayanan2017ensemble}} includes multiple Deterministic~\acrshort{DNN} trained with different seeds. In test-time, the final prediction is calculated from the mean of the list prediction of the ensemble by Equation.~\ref{eq:sampling-based}. Due to aggregates from multiple deterministic models, the latency and total of model weights needed to store will increase linearly w.r.t. the number of models.
\end{itemize}

\subsection{Implementation}\label{apd:implementation}
\textbf{Dataset, source code, and hyper-parameter setting.}
Our source code is available at \href{https://github.com/Angie-Lab-JHU/density_regression}{https://github.com/Angie-Lab-JHU/density\_regression}, including our notebook demo on the toy dataset, scripts to download the benchmark dataset, setup for environment configuration, and our provided code (detail in README.md). All baselines follow the same hyper-parameters setting, data-split, and evaluation technique in training. Specifically, for the Toy-dataset, Benchmark UCI, and monocular depth estimation, we follow \acrshort{EDL}~\citep{amini2020deep}. For Time series weather forecasting, we follow the Time series forecasting Tensorflow tutorial~\citep{tensorflow2015-whitepaper}. Regarding the density function, we use the ``KernelDensity(kernel = 'exponential', metric = "l1")'' for the Toy-dataset, and we reuse the Normalizing Flows architecture following Density-Softmax~\citep{bui2023densitysoftmax} for the remained dataset.

\textbf{Computing system.}
We test our model on three different settings, including (1) a single GPU: NVIDIA Tesla~K80 accelerator-12GB GDDR5 VRAM with 8-CPUs: Intel(R) Xeon(R) Gold 6248R CPU @ 3.00GHz with 8GB RAM per each; (2) a single GPU: NVIDIA RTX~A5000-24564MiB with 8-CPUs: AMD Ryzen Threadripper 3960X 24-Core with 8GB RAM per each; and (3) a single GPU: NVIDIA A100-PCIE-40GB with 8 CPUs: Intel(R) Xeon(R) Gold 6248R CPU @ 3.00GHz with 8GB RAM per each.

\subsection{Demo notebook code for Algorithm~\ref{alg:algorithm}}\label{apd:code}
\begin{tabular}{cc}
\begin{minipage}{0.1\linewidth}
\end{minipage}&
\begin{minipage}{0.96\linewidth}
\begin{minted}
[
frame=lines,
framesep=2mm,
baselinestretch=1,
bgcolor=LightGray,
fontsize=\footnotesize,
linenos
]
{python}
import tensorflow as tf

#Define a features extractor f.
model = tf.keras.Sequential([
	tf.keras.layers.Dense(100, activation = "relu"),
	tf.keras.layers.Dense(100, activation = "relu"),
])
#Define a regressor g.
regressor = tf.keras.layers.Dense(2)

#Define a tf step function to pre-train model w.r.t. Eq. 11.
@tf.function
def pre_train_step(x, y):
	with tf.GradientTape() as tape:
		y_pred = regressor(model(x, training = True), training = True)
		M_ys, M_ymu = tf.split(y_pred, 2, axis = -1)
		log_std = -1/2 * (tf.math.log(2.) + M_ys)
		var = tf.exp(log_std) ** 2
		mean = var * (-2 * M_ymu)
		loss_value = tf.reduce_mean(2 * log_std + ((y - mean) / tf.exp(log_std)) ** 2)

	list_weights = model.trainable_weights + regressor.trainable_weights
	grads = tape.gradient(loss_value, list_weights)
	optimizer.apply_gradients(zip(grads, list_weights))
	return loss_value

#Define a tf step function to re-update the regressor by feature density model w.r.t. Eq. 15.
@tf.function
def train_step(z, y, loglikelihood):
	with tf.GradientTape() as tape:
		y_pred = regressor(z, training = True)
		M_ys, M_ymu = tf.split(y_pred, 2, axis = -1)
		log_std = -1/2 * (tf.math.log(2.) + loglikelihood + M_ys)
		var = tf.exp(log_std) ** 2
		mean = var * (-2 * tf.exp(loglikelihood) * M_ymu)
		loss_value = tf.reduce_mean(2 * log_std + ((y - mean) / tf.exp(log_std)) ** 2)

	list_weights = regressor.trainable_weights
	grads = tape.gradient(loss_value, list_weights)
	optimizer.apply_gradients(zip(grads, list_weights))
	return loss_value

#Define a tf step function to make inference w.r.t. Cor. 3.2.
@tf.function
def test_step(z, loglikelihood):
	y_pred = regressor(z, training = False)
	M_ys, M_ymu = tf.split(y_pred, 2, axis = -1)
	log_std = -1/2 * (tf.math.log(2.) + loglikelihood + M_ys)
	var = tf.exp(log_std) ** 2
	mean = var * (-2 * tf.exp(loglikelihood) * M_ymu)
	y_pred = tf.concat([mean, var], 1)
	return y_pred
\end{minted}
\end{minipage}
\end{tabular}
\section{Additional results}\label{apd:results}
In this appendix, we collect additional results that we deferred from the main paper. 
\vspace{-0.1in}
\subsection{Monocular depth estimation}\label{apd:depth}
\vspace{-0.1in}
\begin{figure}[ht!]
    \centering
    \includegraphics[width=1.0\linewidth]{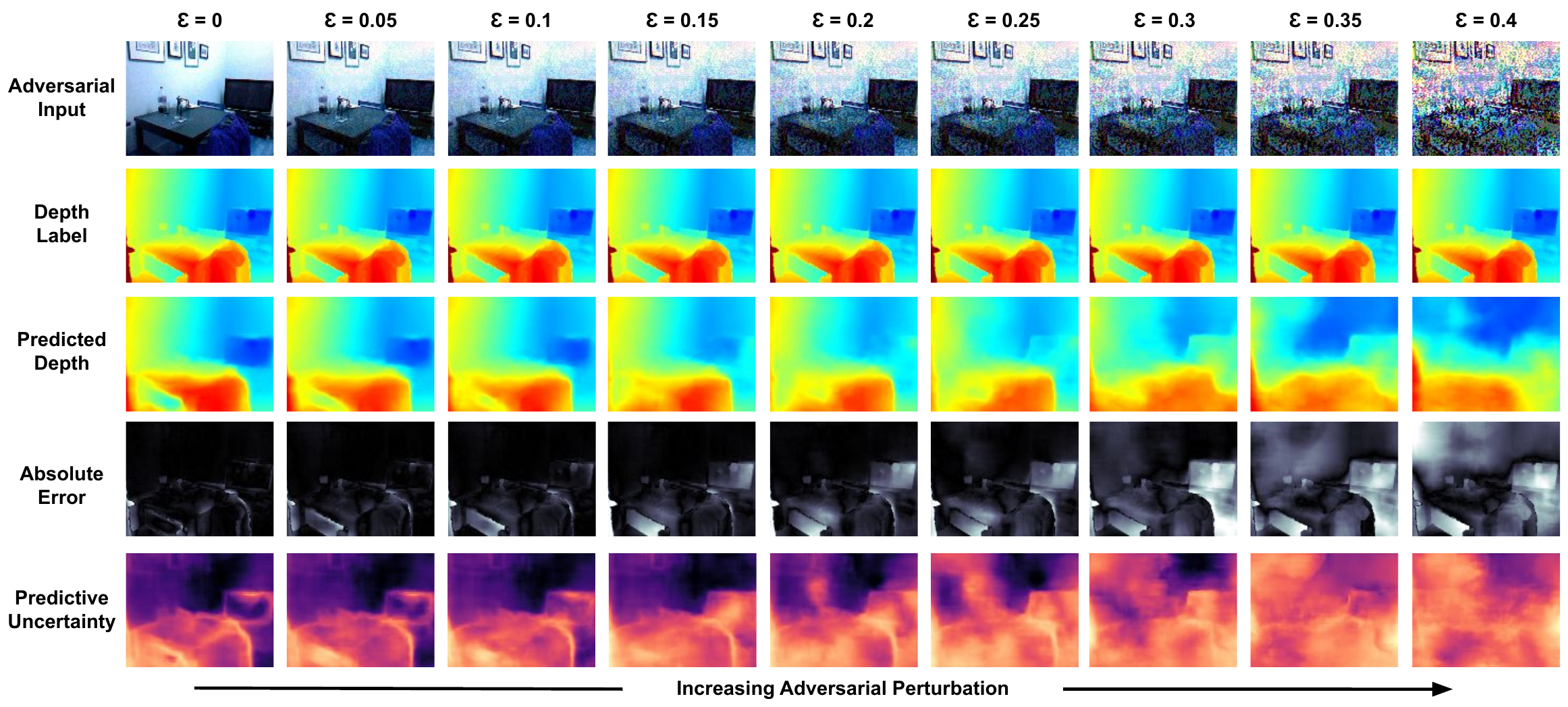}
    \vspace{-0.2in}
    \caption{Visualization of an example with different shift intensities and Density-Regression's performance, including predicted depth, absolute prediction error, and predictive uncertainty per pixel. Density-Regression confidence on \acrshort{IID} ($\epsilon = 0$), and the confidence decreases w.r.t. the increasing of the shift intensities on \acrshort{OOD}.}
    \vspace{-0.1in}
\end{figure}

\vspace{-0.1in}
\subsection{Time-series}\label{apd:time_series}
\vspace{-0.1in}
\begin{figure}[ht!]
    \centering
    \includegraphics[width=0.99\linewidth]{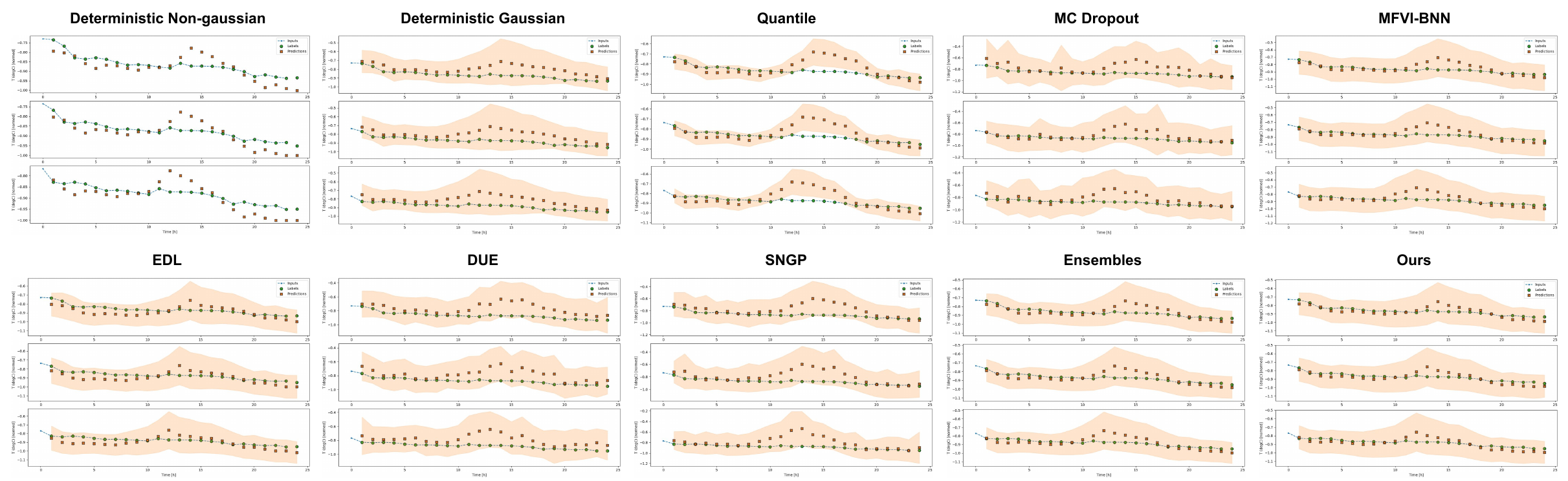}
    \vspace{-0.1in}
    \caption{Comparison between methods in terms of temperature in Celsius (normalized) for every hour on three days on \acrshort{IID} data. Our Density-Regression is still certain on \acrshort{IID}.}
    \vspace{-0.1in}
\end{figure}

\begin{figure}[ht!]
    \centering
    \includegraphics[width=0.99\linewidth]{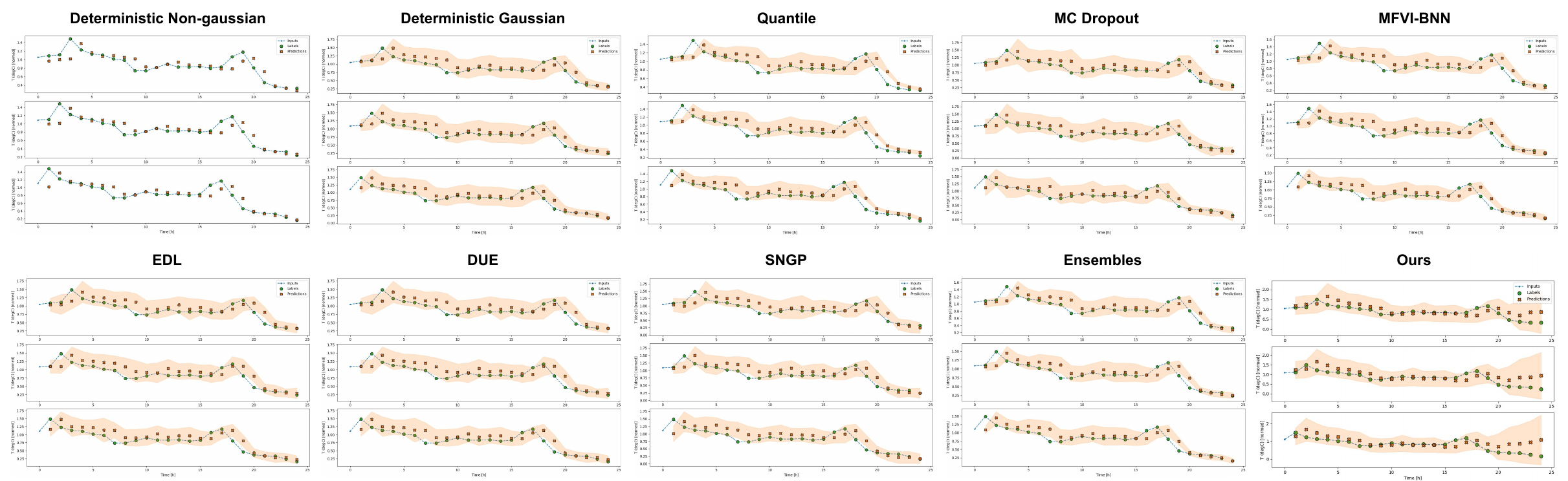}
    \vspace{-0.1in}
    \caption{Comparison between methods in terms of temperature in Celsius (normalized) for every hour on three days on \acrshort{OOD} data. Our Density-Regression is more uncertain on \acrshort{OOD}.}
    \vspace{-0.6in}
\end{figure}

\begin{table}[ht!]
  \begin{minipage}[t!]{1.0\textwidth}
    \centering
    \includegraphics[width=0.8\linewidth]{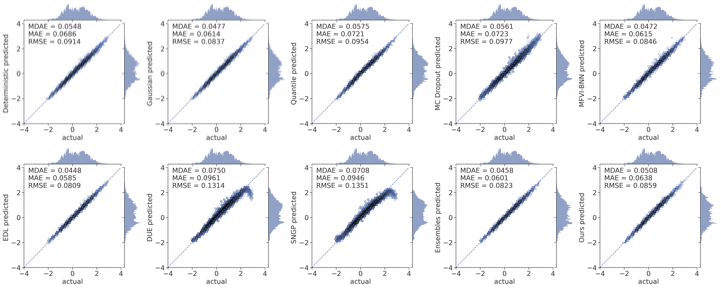}
  \end{minipage}
  \begin{minipage}[t!]{1.0\textwidth}
    \centering
    \includegraphics[width=0.8\linewidth]{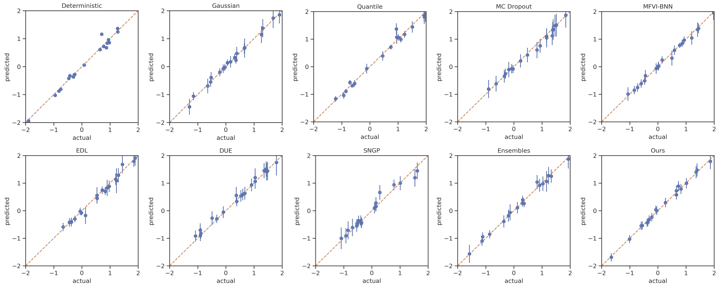}
  \end{minipage}
  \begin{minipage}[t!]{1.0\textwidth}
    \centering
    \includegraphics[width=0.8\linewidth]{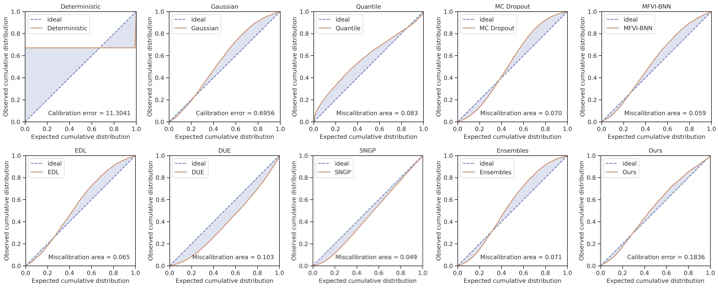}
  \end{minipage}
  \begin{minipage}[t!]{1.0\textwidth}
    \centering
    \includegraphics[width=0.8\linewidth]{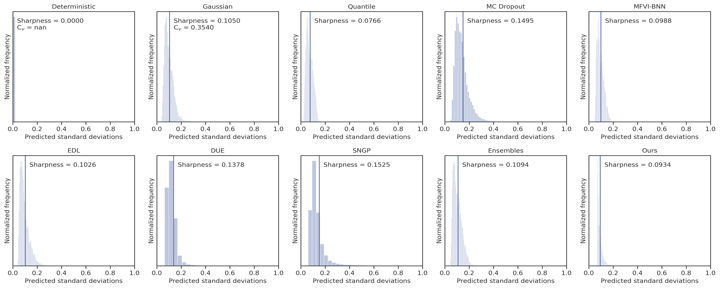}
    \captionof{figure}{Detailed visualization of regression error, parity plot, calibration curve, and distribution plots with sharpness on the \acrshort{IID} Time series weather forecasting.}
  \end{minipage}
\end{table}

\begin{table}[ht!]
  \begin{minipage}[t!]{1.0\textwidth}
    \centering
    \includegraphics[width=0.8\linewidth]{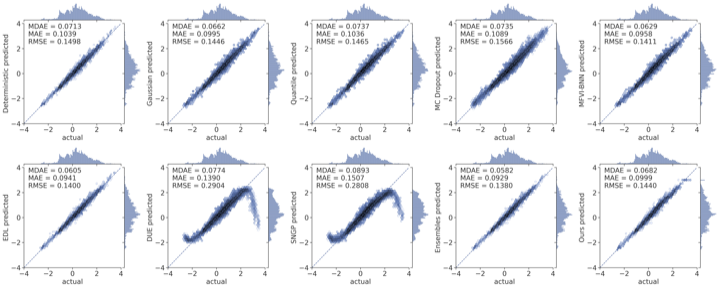}
  \end{minipage}
  \begin{minipage}[t!]{1.0\textwidth}
    \centering
    \includegraphics[width=0.8\linewidth]{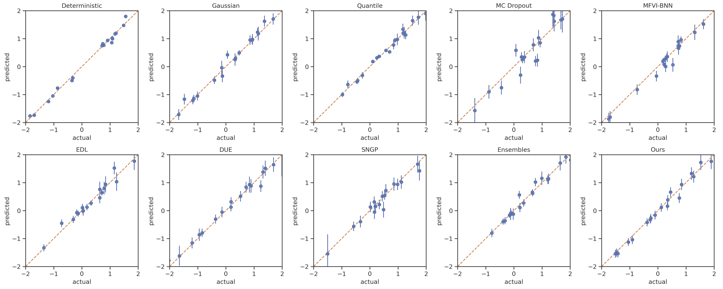}
  \end{minipage}
  \begin{minipage}[t!]{1.0\textwidth}
    \centering
    \includegraphics[width=0.8\linewidth]{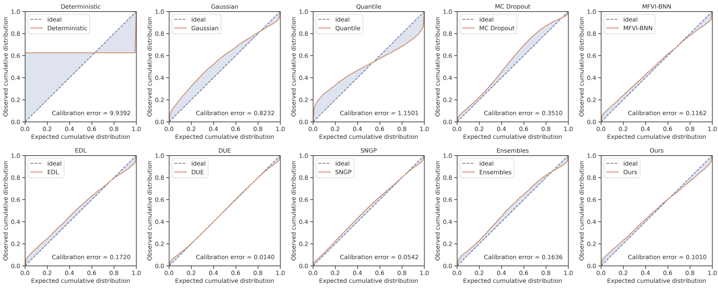}
  \end{minipage}
  \begin{minipage}[t!]{1.0\textwidth}
    \centering
    \includegraphics[width=0.8\linewidth]{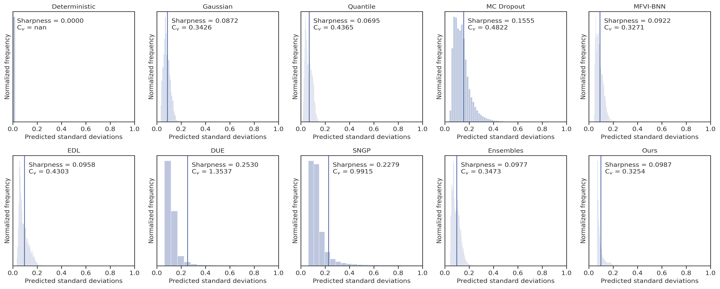}
    \captionof{figure}{Detailed visualization of regression error, parity plot, calibration curve, and distribution plots with sharpness on the \acrshort{OOD} Time series weather forecasting.}
  \end{minipage}
\end{table}

\clearpage
\subsection{Benchmark UCI}\label{apd:uci}
\begin{table}[ht!]
\begin{minipage}[t!]{0.5\textwidth}
    \caption{UCI: Boston housing}
    \centering
    \scalebox{0.66}{
    \begin{tabular}{lcccc}
    \toprule
    \textbf{Method} & \textbf{NLL ($\downarrow$)} & \textbf{RMSE ($\downarrow$)} & \textbf{Cal ($\downarrow$)} & \textbf{Sharp ($\downarrow$)}\\
    \midrule
    Deterministic & 2.64 $\pm$ 0.26 & 3.05 $\pm$ 0.21 & 0.39 $\pm$ 0.25 & 0.45 $\pm$ 0.08\\
    Quantile & 2.73 $\pm$ 0.75 & 3.03 $\pm$ 0.19 & 0.95 $\pm$ 0.49 & 0.30 $\pm$ 0.02\\
    MC~Dropout & 2.42 $\pm$ 0.09 & 3.03 $\pm$ 0.26 & 0.54 $\pm$ 0.14 & 0.49 $\pm$ 0.06\\
    MFVI-BNN & 6.06 $\pm$ 0.51 & 10.9 $\pm$ 1.82 & 6.51 $\pm$ 0.71 & 17.2 $\pm$ 5.58\\
    EDL & 2.35 $\pm$ 0.06 & 3.02 $\pm$ 0.21 & 0.41 $\pm$ 0.19 & 0.76 $\pm$ 0.23\\
    SNGP & 2.37 $\pm$ 0.06 & 4.62 $\pm$ 0.33 & 0.42 $\pm$ 0.42 & 0.49 $\pm$ 0.03\\
    DUE & 2.35 $\pm$ 0.09 & 3.21 $\pm$ 0.28 & 0.41 $\pm$ 0.28 & 0.44 $\pm$ 0.05\\
    Ensembles & \textbf{2.30 $\pm$ 0.07} & \textbf{2.91 $\pm$ 0.11} & 0.25 $\pm$ 0.16 & 0.40 $\pm$ 0.03\\
    \textbf{Ours} & 2.46 $\pm$ 0.04 & 2.93 $\pm$ 0.11 & \textbf{0.22 $\pm$ 0.12} & \textbf{0.37 $\pm$ 0.01}\\
    \bottomrule
    \end{tabular}}
    \label{tab:boston}
\end{minipage}
\begin{minipage}{0.5\textwidth}
    \caption{UCI: Concrete}
    \centering
    \scalebox{0.66}{
    \begin{tabular}{lcccc}
    \toprule
    \textbf{Method} & \textbf{NLL ($\downarrow$)} & \textbf{RMSE ($\downarrow$)} & \textbf{Cal ($\downarrow$)} & \textbf{Sharp ($\downarrow$)}\\
    \midrule
    Deterministic & 3.02 $\pm$ 0.09 & 5.58 $\pm$ 0.92 & 0.71 $\pm$ 0.78 & 0.39 $\pm$ 0.06\\
    Quantile & 3.24 $\pm$ 0.14 & 5.94 $\pm$ 0.54 & 1.04 $\pm$ 1.22 & 0.30 $\pm$ 0.03\\
    MC~Dropout & 3.23 $\pm$ 0.05 & 6.33 $\pm$ 0.39 & 0.48 $\pm$ 0.38 & 0.50 $\pm$ 0.07\\
    MFVI-BNN & 6.93 $\pm$ 0.18 & 17.2 $\pm$ 4.51 & 8.64 $\pm$ 1.15 & 20.5 $\pm$ 6.13 \\
    EDL & 3.03 $\pm$ 0.14 & 5.18 $\pm$ 0.5 & 2.24 $\pm$ 0.34 & 3.23 $\pm$ 2.26\\
    SNGP & 3.45 $\pm$ 0.07 & 7.59 $\pm$ 0.57 & 0.63 $\pm$ 0.61 & 0.56 $\pm$ 0.04\\
    DUE & 3.47 $\pm$ 0.07 & 7.82 $\pm$ 0.58 & 0.48 $\pm$ 0.45 & 0.56 $\pm$ 0.06\\
    Ensembles & \textbf{2.93 $\pm$ 0.04} & \textbf{4.82 $\pm$ 0.18} & 0.44 $\pm$ 0.24 & 0.37 $\pm$ 0.03\\
    \textbf{Ours} & 2.97 $\pm$ 0.08 & 4.94 $\pm$ 0.48 & 0\textbf{.23 $\pm$ 0.33} & \textbf{0.30 $\pm$ 0.01} \\
    \bottomrule
    \end{tabular}}
    \label{tab:concrete}
\end{minipage}
\end{table}

\begin{table}[ht!]
\begin{minipage}[t!]{0.5\textwidth}
    \caption{UCI: Energy}
    \centering
    \scalebox{0.66}{
    \begin{tabular}{lcccc}
    \toprule
    \textbf{Method} & \textbf{NLL ($\downarrow$)} & \textbf{RMSE ($\downarrow$)} & \textbf{Cal ($\downarrow$)} & \textbf{Sharp ($\downarrow$)}\\
    \midrule
    Deterministic & 1.95 $\pm$ 0.27 & 2.18 $\pm$ 0.11 & 1.10 $\pm$ 1.40 & 0.21 $\pm$ 0.01\\
    Quantile & 1.91 $\pm$ 0.51 & 2.21 $\pm$ 0.16 & 1.15 $\pm$ 1.01 & 0.18 $\pm$ 0.02\\
    MC~Dropout & 2.11 $\pm$ 0.08 & 2.94 $\pm$ 0.09 & 0.59 $\pm$ 0.64 & 0.36 $\pm$ 0.05\\
    MFVI-BNN & 3.14 $\pm$ 0.57 & 3.00 $\pm$ 0.15 & 1.77 $\pm$ 1.40 & 6.99 $\pm$ 2.31 \\
    EDL & 1.64 $\pm$ 0.12 & 2.23 $\pm$ 0.13 & 3.05 $\pm$ 2.33 & 12.1 $\pm$ 4.12\\
    SNGP & 1.77 $\pm$ 0.11 & 2.18 $\pm$ 0.18 & 0.82 $\pm$ 0.82 & 0.27 $\pm$ 0.02\\
    DUE & 1.90 $\pm$ 0.09 & 2.61 $\pm$ 0.24 & \textbf{0.57 $\pm$ 0.64} & \textbf{0.29 $\pm$ 0.04}\\
    Ensembles & \textbf{1.53 $\pm$ 0.07} & \textbf{2.14 $\pm$ 0.07} & 0.67 $\pm$ 0.54 & 0.24 $\pm$ 0.02\\
    \textbf{Ours} & 1.56 $\pm$ 0.13 & 2.15 $\pm$ 0.09 & 0.77 $\pm$ 0.84 & 0.22 $\pm$ 0.02\\
    \bottomrule
    \end{tabular}}
    \label{tab:energy}
\end{minipage}
\begin{minipage}{0.5\textwidth}
    \caption{UCI: Kin8nm}
    \centering
    \scalebox{0.66}{
    \begin{tabular}{lcccc}
    \toprule
    \textbf{Method} & \textbf{NLL ($\downarrow$)} & \textbf{RMSE ($\downarrow$)} & \textbf{Cal ($\downarrow$)} & \textbf{Sharp ($\downarrow$)}\\
    \midrule
    Deterministic & -1.17 $\pm$ 0.03 & 0.08 $\pm$ 0.00 & 0.13 $\pm$ 0.15 & 0.30 $\pm$ 0.02\\
    Quantile & -1.08 $\pm$ 0.05 & 0.08 $\pm$ 0.00 & 0.54 $\pm$ 0.42 & 0.26 $\pm$ 0.02\\
    MC~Dropout & -0.91 $\pm$ 0.02 & 0.11 $\pm$ 0.01 & 0.31 $\pm$ 0.17 & 0.53 $\pm$ 0.02\\
    MFVI-BNN & -0.03 $\pm$ 0.11 & 0.18 $\pm$ 0.03 & 1.58 $\pm$ 1.57 & 5.12 $\pm$ 2.00 \\
    EDL & -1.07 $\pm$ 0.05 & 0.08 $\pm$ 0.00 & 0.83 $\pm$ 0.00 & 9.99 $\pm$ 0.30 \\
    SNGP & -1.05 $\pm$ 0.02 & 0.09 $\pm$ 0.00 & 0.18 $\pm$ 0.09 & 0.39 $\pm$ 0.01\\
    DUE & -1.02 $\pm$ 0.02 & 0.09 $\pm$ 0.00 & 0.16 $\pm$ 0.08 & 0.40 $\pm$ 0.01\\
    Ensembles & \textbf{-1.25 $\pm$ 0.02} & 0.08 $\pm$ 0.00 & 0.40 $\pm$ 0.33 & 0.33 $\pm$ 0.02\\
    \textbf{Ours} & -1.22 $\pm$ 0.02 & 0.08 $\pm$ 0.00 & \textbf{0.11 $\pm$ 0.19} & \textbf{0.30 $\pm$ 0.01}\\
    \bottomrule
    \end{tabular}}
    \label{tab:kin8nm}
\end{minipage}
\end{table}

\begin{table}[ht!]
\begin{minipage}[t!]{0.5\textwidth}
    \caption{UCI: Naval propulsion plant}
    \centering
    \scalebox{0.66}{
    \begin{tabular}{lcccc}
    \toprule
    \textbf{Method} & \textbf{NLL ($\downarrow$)} & \textbf{RMSE ($\downarrow$)} & \textbf{Cal ($\downarrow$)} & \textbf{Sharp ($\downarrow$)}\\
    \midrule
    Deterministic & -4.93 $\pm$ 0.62 & 0.00 $\pm$ 0.00 & 2.58 $\pm$ 1.39 & 0.30 $\pm$ 0.05\\
    Quantile & -4.79 $\pm$ 0.21 & 0.00 $\pm$ 0.00 & 3.04 $\pm$ 1.12 & 0.39 $\pm$ 0.02\\
    MC~Dropout & -4.69 $\pm$ 0.07 & 0.00 $\pm$ 0.00 & 0.66 $\pm$ 0.34 & 0.52 $\pm$ 0.05\\
    MFVI-BNN & -4.07 $\pm$ 0.11 & 0.03 $\pm$ 0.01 & 4.10 $\pm$ 0.42 & 1.61 $\pm$ 0.11 \\
    EDL & -5.37 $\pm$ 0.27 & 0.00 $\pm$ 0.00 & 8.25 $\pm$ 0.00 & 100 $\pm$ 0.00 \\
    SNGP & -4.05 $\pm$ 0.08 & 0.00 $\pm$ 0.00 & \textbf{0.54 $\pm$ 0.49} & 0.74 $\pm$ 0.02\\
    DUE & -3.89 $\pm$ 0.10 & 0.01 $\pm$ 0.00 & 0.71 $\pm$ 0.44 & 0.77 $\pm$ 0.03\\
    Ensembles & \textbf{-5.68 $\pm$ 0.17} & 0.00 $\pm$ 0.00& 2.52 $\pm$ 1.04 & \textbf{0.30 $\pm$ 0.03}\\
    \textbf{Ours} & -5.42 $\pm$ 0.09 & 0.00 $\pm$ 0.00 & 2.29 $\pm$ 0.96 & 0.38 $\pm$ 0.01\\
    \bottomrule
    \end{tabular}}
    \label{tab:naval}
\end{minipage}
\begin{minipage}{0.5\textwidth}
    \caption{UCI: Power plant}
    \centering
    \scalebox{0.66}{
    \begin{tabular}{lcccc}
    \toprule
    \textbf{Method} & \textbf{NLL ($\downarrow$)} & \textbf{RMSE ($\downarrow$)} & \textbf{Cal ($\downarrow$)} & \textbf{Sharp ($\downarrow$)}\\
    \midrule
    Deterministic & 2.85 $\pm$ 0.03 & 4.03 $\pm$ 0.08 & 0.19 $\pm$ 0.18 & 0.24 $\pm$ 0.02\\
    Quantile & 3.03 $\pm$ 0.03 & 4.38 $\pm$ 0.05 & 0.34 $\pm$ 0.11 & 0.19 $\pm$ 0.01\\
    MC~Dropout & 2.85 $\pm$ 0.02 & 4.22 $\pm$ 0.08 & 0.22 $\pm$ 0.24 & 0.27 $\pm$ 0.01\\
    MFVI-BNN & 4.02 $\pm$ 0.10 & 4.55 $\pm$ 0.31 & 10.7 $\pm$ 7.58 & 9.11 $\pm$ 3.04 \\
    EDL & 2.82 $\pm$ 0.02 & 4.07 $\pm$ 0.08 & 8.21 $\pm$ 0.07 & 9.61 $\pm$ 6.72\\
    SNGP & 2.81 $\pm$ 0.01 & 4.10 $\pm$ 0.06 & 0.12 $\pm$ 0.09 & 0.25 $\pm$ 0.01\\
    DUE & 2.84 $\pm$ 0.02 & 4.19 $\pm$ 0.08 & 0.18 $\pm$ 0.26 & 0.26 $\pm$ 0.01\\
    Ensembles & \textbf{2.78 $\pm$ 0.01} & \textbf{3.99 $\pm$ 0.05} & 0.18 $\pm$ 0.18 & 0.23 $\pm$ 0.01\\
    \textbf{Ours} & 2.80 $\pm$ 0.02 & 4.02 $\pm$ 0.06 & \textbf{0.11 $\pm$ 0.12} & \textbf{0.22 $\pm$ 0.02}\\
    \bottomrule
    \end{tabular}}
    \label{tab:power}
\end{minipage}
\end{table}

\begin{table}[ht!]
\begin{minipage}[t!]{0.5\textwidth}
    \caption{UCI: Protein}
    \centering
    \scalebox{0.66}{
    \begin{tabular}{lcccc}
    \toprule
    \textbf{Method} & \textbf{NLL ($\downarrow$)} & \textbf{RMSE ($\downarrow$)} & \textbf{Cal ($\downarrow$)} & \textbf{Sharp ($\downarrow$)}\\
    \midrule
    Deterministic & 2.90 $\pm$ 0.02 & 4.63 $\pm$ 0.07 & 0.17 $\pm$ 0.09 & 0.76 $\pm$ 0.02\\
    Quantile & 3.19 $\pm$ 0.15 & 5.19 $\pm$ 0.04 & 3.26 $\pm$ 0.21 & 0.60 $\pm$ 0.01 \\
    MC~Dropout & 2.94 $\pm$ 0.07 & 4.85 $\pm$ 0.02 & 0.23 $\pm$ 0.08 & 3.76 $\pm$ 1.01\\
    MFVI-BNN & 3.68 $\pm$ 0.53 & 5.97 $\pm$ 0.55 & 3.54 $\pm$ 1.87 & 4.19 $\pm$ 1.44 \\
    EDL & 3.20 $\pm$ 0.22 & 5.01 $\pm$ 0.45 & 3.50 $\pm$ 0.10 & 9.16 $\pm$ 3.22\\
    SNGP & \textbf{2.82 $\pm$ 0.01} & \textbf{4.57 $\pm$ 0.02} & 0.14 $\pm$ 0.04 & 0.78 $\pm$ 0.01\\
    DUE & 2.86 $\pm$ 0.02 & 4.69 $\pm$ 0.04 & 0.16 $\pm$ 0.07 & 0.79 $\pm$ 0.02\\
    Ensembles & \textbf{2.83 $\pm$ 0.01} & \textbf{4.58 $\pm$ 0.01} & 0.12 $\pm$ 0.04 & 0.79 $\pm$ 0.02\\
    \textbf{Ours} & 2.84 $\pm$ 0.01 & 4.61 $\pm$ 0.02 & \textbf{0.10 $\pm$ 0.04} & \textbf{0.73 $\pm$ 0.01}\\
    \bottomrule
    \end{tabular}}
    \label{tab:protein}
\end{minipage}
\begin{minipage}{0.5\textwidth}
    \caption{UCI: Yacht}
    \centering
    \scalebox{0.66}{
    \begin{tabular}{lcccc}
    \toprule
    \textbf{Method} & \textbf{NLL ($\downarrow$)} & \textbf{RMSE ($\downarrow$)} & \textbf{Cal ($\downarrow$)} & \textbf{Sharp ($\downarrow$)}\\
    \midrule
    Deterministic & 1.29 $\pm$ 0.28 & 2.29 $\pm$ 0.56 & 0.93 $\pm$ 0.66 & 0.15 $\pm$ 0.03\\
    Quantile & 2.58 $\pm$ 0.93 & 3.90 $\pm$ 0.26 & 2.13 $\pm$ 1.72 & 0.17 $\pm$ 0.04 \\
    MC~Dropout & 1.82 $\pm$ 0.13 & 2.76 $\pm$ 0.38 & 2.19 $\pm$ 0.65 & 0.28 $\pm$ 0.03\\
    MFVI-BNN & 9.40 $\pm$ 1.14 & 22.6 $\pm$ 6.26 & 7.68 $\pm$ 0.39 & 28.9 $\pm$ 4.01 \\
    EDL & 1.12 $\pm$ 0.16 & 2.48 $\pm$ 0.53 & 1.42 $\pm$ 1.25 & 0.30 $\pm$ 0.07\\
    SNGP & 1.16 $\pm$ 0.17 & 6.53 $\pm$ 0.49 & 1.29 $\pm$ 0.87 & 0.49 $\pm$ 0.05\\
    DUE & 1.22 $\pm$ 0.26 & 4.56 $\pm$ 1.23 & 1.38 $\pm$ 1.02 & 0.40 $\pm$ 0.08\\
    Ensembles & \textbf{1.18 $\pm$ 0.10} & 2.22 $\pm$ 0.21 & 2.02 $\pm$ 0.32 & 0.17 $\pm$ 0.02\\
    \textbf{Ours} & 1.24 $\pm$ 0.19 & \textbf{2.20 $\pm$ 0.42} & \textbf{0.87 $\pm$ 0.56} & \textbf{0.15 $\pm$ 0.03}\\
    \bottomrule
    \end{tabular}}
    \label{tab:yacht}
\end{minipage}
\end{table}

\begin{table}[ht!]
\begin{minipage}[t!]{0.5\textwidth}
    \caption{UCI: Year Prediction MSD}
    \centering
    \scalebox{0.66}{
    \begin{tabular}{lcccc}
    \toprule
    \textbf{Method} & \textbf{NLL ($\downarrow$)} & \textbf{RMSE ($\downarrow$)} & \textbf{Cal ($\downarrow$)} & \textbf{Sharp ($\downarrow$)}\\
    \midrule
    Deterministic & 3.49 $\pm$ 0.01 & 9.65 $\pm$ 0.07 & 0.06 $\pm$ 0.04 & 0.46 $\pm$ 0.02\\
    Quantile & 3.45 $\pm$ 0.01 & 9.46 $\pm$ 0.04 & 0.12 $\pm$ 0.02 & 0.37 $\pm$ 0.01\\
    MC~Dropout & 3.50 $\pm$ 0.01 & 10.3 $\pm$ 0.08 & 0.07 $\pm$ 0.01 & 57.5 $\pm$ 4.49 \\
    MFVI-BNN & 4.21 $\pm$ 0.10 & 11.5 $\pm$ 0.97 & 0.25 $\pm$ 0.04 & 26.1 $\pm$ 5.12\\
    EDL & 4.09 $\pm$ 0.22 & 11.2 $\pm$ 0.90 & 0.22 $\pm$ 0.08 & 21.1 $\pm$ 8.05 \\
    SNGP & 3.52 $\pm$ 0.01 & 11.2 $\pm$ 0.17 & 0.12 $\pm$ 0.08 & 0.58 $\pm$ 0.01\\
    DUE & 3.53 $\pm$ 0.06 & 11.3 $\pm$ 0.81 & 0.10 $\pm$ 0.03 & 0.59 $\pm$ 0.04\\
    Ensembles & \textbf{3.35 $\pm$ 0.01} & \textbf{9.31 $\pm$ 0.03} & 0.06 $\pm$ 0.03 & 0.48 $\pm$ 0.02\\
    \textbf{Ours} & 3.39 $\pm$ 0.00 & 9.57 $\pm$ 0.02 & \textbf{0.04 $\pm$ 0.02} & \textbf{0.40 $\pm$ 0.01}\\
    \bottomrule
    \end{tabular}}
    \label{tab:year}
\end{minipage}
\begin{minipage}{0.5\textwidth}
    \vspace{-0.2in}
    \begin{center}
      \includegraphics[width=0.8\linewidth]{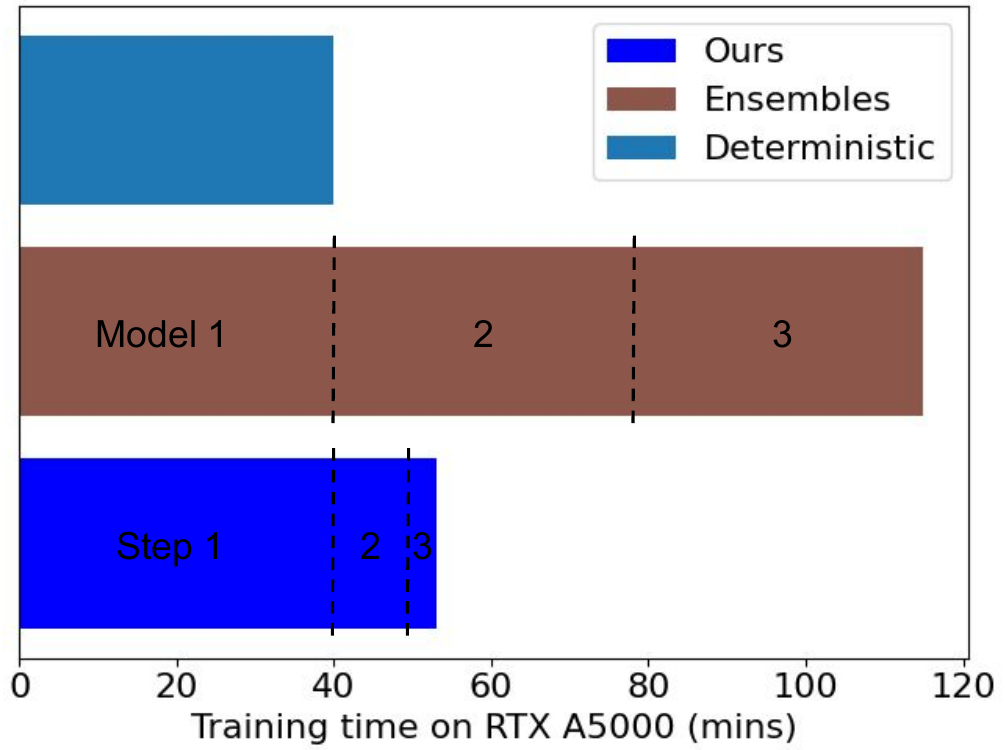}
    \end{center}
    \vspace{-0.1in}
    \caption{Training time comparison between Deterministic, Ensembles, and Ours on the depth estimation setting.} 
    \label{fig:training-time}
    \vspace{-1in}
\end{minipage}
\end{table}

\end{document}